\documentclass[conference]{IEEEtran}
\usepackage[english]{babel}

\usepackage[linesnumbered,ruled]{algorithm2e}
\usepackage{algpseudocode}
\usepackage{color}
\definecolor{g}{gray}{0.6}

\usepackage{amsthm}
\newtheorem{theorem}{Theorem}
\newtheorem{lemma}{Lemma}

\usepackage{multirow}
\usepackage{subcaption}
\usepackage{graphicx}
\usepackage{longtable}

\usepackage{soul}

\usepackage{soul}
\usepackage{booktabs}
\usepackage{arydshln}
\usepackage{amsmath}
\usepackage{amssymb}
\usepackage{float}
\usepackage{supertabular}

\usepackage{hyperref}

\definecolor{orange}{rgb}{0.858, 0.5, 0.1}
\definecolor{green}{rgb}{0.1, 0.5, 0.1}
\definecolor{blue}{rgb}{0.1, 0.3, 0.7}

\ifCLASSINFOpdf
\else
\fi
\begin{document}
%
\title{Denoising Adversarial Autoencoders}

\author{\IEEEauthorblockN{Antonia Creswell}
\IEEEauthorblockA{BICV\\
Imperial College London\\
\\
Email: ac2211@ic.ac.uk}
\and
\IEEEauthorblockN{Anil Anthony Bharath}
\IEEEauthorblockA{BICV\\
Imperial College London}}


%


\maketitle

\begin{abstract}
Unsupervised learning is of growing interest because it unlocks the potential held in vast amounts of unlabelled data to learn useful representations for inference. Autoencoders, a form of generative model, may be trained by learning to reconstruct unlabelled input data from a latent representation space. More robust representations may be produced by an autoencoder if it learns to recover clean input samples from corrupted ones. Representations may be further improved by introducing regularisation during training to shape the distribution of the encoded data in the latent space. We suggest {\it denoising adversarial autoencoders}, which combine denoising and regularisation, shaping the distribution of latent space using adversarial training. We introduce a novel analysis that shows how denoising may be incorporated into the training and sampling of adversarial autoencoders. Experiments are performed to assess the contributions that denoising makes to the learning of representations for classification and sample synthesis. Our results suggest that autoencoders trained using a denoising criterion achieve higher classification performance, and can synthesise samples that are more consistent with the input data than those trained without a corruption process.\footnote{This work has been submitted to the IEEE for possible publication. Copyright may be transferred without notice, after which this version may no longer be accessible.}
\end{abstract}


%
\IEEEpeerreviewmaketitle

Modelling and drawing data samples from complex, high-dimensional distributions is challenging. {\em Generative models} may be used to capture underlying statistical structure from real-world data. A good generative model is not only able to draw samples from the distribution of data being modelled, but should also be useful for inference. 

Modelling complicated distributions may be made easier by learning the parameters of  conditional probability distributions that map intermediate, \textit{latent}, \cite{bachman2015variational} variables from simpler distributions to more complex ones \cite{bengio2014deep}. Often, the intermediate representations that are learned can be used for tasks such as retrieval or classification \cite{radford2015unsupervised,makhzani2015adversarial,salimans2016improved,vincent2008extracting}.

Typically, to train a model for classification, a deep convolutional neural network may be constructed, demanding large labelled datasets to achieve high accuracy \cite{krizhevsky2012imagenet}. Large labelled datasets may be expensive or difficult to obtain for some tasks. However, many state-of-the-art generative models can be trained without labelled datasets \cite{goodfellow2014generative, radford2015unsupervised, kingma2013auto, im2017denoising}. For example, autoencoders learn a generative model, referred to as a \textit{decoder}, by recovering inputs from corrupted \cite{vincent2008extracting, im2017denoising, bengio2013generalized} or \textit{encoded} \cite{kingma2013auto} versions of themselves. 


Two broad approaches to learning state-of-the-art generative models that do not require labelled training data include: 1) introduction of a denoising criterion \cite{vincent2010stacked,bengio2013generalized, vincent2008extracting} -- where the model learns to reconstruct clean samples from corrupted ones; 2) regularisation of the latent space to match a prior \cite{kingma2013auto, makhzani2015adversarial}; for the latter, the priors take a simple form, such as multivariate normal distributions.


\begin{figure}
    \centering
    \includegraphics[width=1.0\columnwidth]{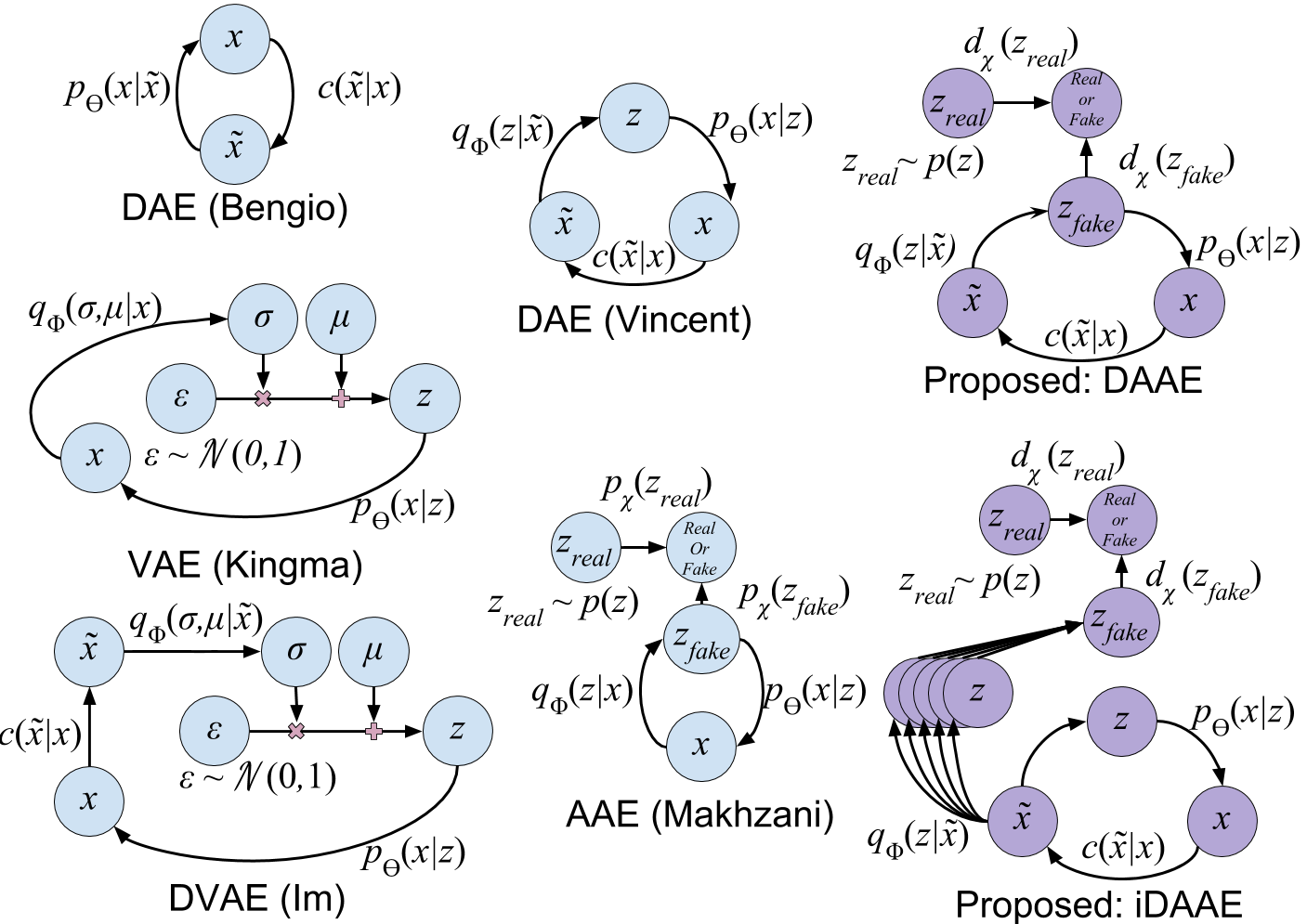}
    \caption{\textbf{Comparison of autoencoding models:} Previous works include Denoising Autoencoders (DAE) \cite{bengio2013generalized, vincent2008extracting}, Variational Autoencoders (VAE) \cite{kingma2013auto}, Adversarial Autoencoders (AAE) \cite{makhzani2015adversarial} and Denoising Variational Autoencoders (DVAE) \cite{im2017denoising}. Our contributions are the DAAE and the iDAAE models. Arrows in this diagram represent mappings implemented using trained neural networks.}
    \label{autoencoders}
\end{figure}

The denoising variational autoencoder \cite{im2017denoising} combines both denoising and regularisation in a single generative model. However, introducing a denoising criterion makes the variational cost function -- used to match the latent distribution to the prior -- analytically intractable \cite{im2017denoising}. Reformulation of the cost function \cite{im2017denoising}, makes it tractable, but only for certain families of prior and posterior distributions. We propose using adversarial training \cite{goodfellow2014generative} to match the posterior distribution to the prior. Taking this approach expands the possible choices for families of prior and posterior distributions.


When a denoising criterion is introduced to an adversarial autoencoder, we have a choice to either shape the conditional distribution of latent variables given {\em corrupted} samples to match the prior (as done using a variational approach \cite{im2017denoising}), or to shape the {\em full} posterior conditional on the {\em original} data samples to match the prior. Shaping the posterior distribution over corrupted samples does not require
additional sampling during training, but trying to shape the full conditional distribution with respect to the original data samples does. We explore both approaches, using adversarial training to avoid the difficulties posed by analytically intractable cost functions. 

Additionally, a model that has been trained using the posterior conditioned on corrupted data requires an iterative process for synthesising samples, whereas using the full posterior conditioned on the original data does not. Similar challenges exist for the denoising VAE, but were not addressed by Im et al.\cite{im2017denoising}. We analyse and address these challenges for adversarial autoencoders, introducing a novel sampling approach for synthesising samples from trained models.
 
In summary, our contributions include: 1) Two types of denoising adversarial autoencoders, one which is more efficient to train, and one which is more efficient to draw samples from; 2) Methods to draw synthetic data samples from denoising adversarial autoencoders through Markov chain (MC) sampling; 3) An analysis of the quality of features learned with denoising adversarial autoencoders through their application to discriminative tasks.

\section{Background}
\subsection{Autoencoders}
In a supervised learning setting, given a set of training data, $\{(y_i, x_i)\}_{i=1}^N$ we wish to learn a model, $f_\psi(y|x)$ that maximises the likelihood, $\mathrm{E}_{p(y|x)}f_\psi(y|x)$ of the true label, $y$ given an observation, $x$. In the supervised setting, there are many ways to calculate and approximate the likelihood, because there is a ground truth label for every training data sample.

When trying to learn a generative model, $p_\theta(x)$, in the absence of a ground truth, calculating the likelihood of the model under the observed data distribution, $\mathrm{E}_{x\sim p(x)}p_\theta(x)$,  is challenging. Autoencoders introduce a two step learning process, that allows the estimation, $p_\theta(x)$ of $p(x)$ via an auxiliary variable, $z$. The variable $z$ may take many forms and we shall explore several of these in this section. The two step process involves first learning a \textit{probabilistic encoder} \cite{kingma2013auto}, $q_\phi(z|x)$, conditioned on observed samples, and a second \textit{probabilistic decoder} \cite{kingma2013auto}, $p_\theta(x|z)$, conditioned on the auxiliary variables. Using the probabilistic encoder, we may form a training dataset, $\{(z_i, x_i)\}_{i=1}^N$ where $x_i$ is the ground truth output for $x \sim p(x|z_i)$ with the input being $z_i \sim q_\phi(z|x_i)$. The probabilistic decoder, $p_\theta(x|z)$, may then be trained on this dataset in a supervised fashion. By sampling $p_\theta(x|z)$ conditioning on suitable $z$'s we may obtain a joint distribution, $p_\theta(x,z)$, which may be marginalised by integrating over all $z$ to obtain, to $p_\theta(x)$.


In some situations the encoding distribution is chosen rather than learned \cite{bengio2013generalized}, in other situations the encoder and decoder are learned simultaneously \cite{kingma2013auto, makhzani2015adversarial, im2017denoising}.




\subsection{Denoising Autoencoders (DAEs)}
Bengio et al. \cite{bengio2013generalized} treat the encoding process as a local corruption process, that does not need to be learned. The corruption process, defined as $c(\tilde{x}|x)$ where $\tilde{x}$, the corrupted $x$ is the auxiliary variable (instead of $z$). The decoder, $p_\theta(x|\tilde{x})$, is therefore trained on the data pairs, $\{(\tilde{x}_i, x_i)\}_{i=1}^N$.

By using a local corruption process (e.g. additive white Gaussian noise \cite{bengio2013generalized}), both $\tilde{x}$ and $x$ have the same number of dimensions and are close to each-other. This makes it very easy to learn $p_\theta(x|\tilde{x})$. Bengio et al. \cite{bengio2013generalized} shows how the learned model may be sampled using an iterative process, but does not explore how representations learned by the model may transfer to other applications such as classification.

Hinton et al. \cite{hinton2006reducing} show that when auxiliary variables of an autoencoder have lower dimension than the observed data, the encoding model learns representations that may be useful for tasks such as classification and retrieval. 

Rather than treating the corruption process, $c(\tilde{x},x)$, as an encoding process \cite{bengio2013generalized} -- missing out on potential benefits of using a lower dimensional auxiliary variable -- Vincent et al. \cite{vincent2008extracting, vincent2010stacked} learn an encoding distribution, $q_\phi(z|\tilde{x})$, conditional on corrupted samples. The decoding distribution, $p_\theta(x|z)$ learns to reconstruct images from encoded, corrupted images, see the DAEs in Figure \ref{autoencoders}. Vincent et al. \cite{vincent2010stacked, vincent2008extracting} show that compared to regular autoencoders, denoising autoencoders learn representations that are more useful and robust for tasks such as classification. Parameters, $\phi$ and $\theta$ are learned simultaneously by minimisng the reconstruction error for the training set, $\{(\tilde{x_i},x_i)\}_{i=1}^M$, which does not include $z_i$. The ground truth $z_i$ for a given $\tilde{x}_i$ is unknown. The form of the distribution over $z$, to which $x$ samples are mapped, $p_\theta(z)$ is also unknown - making it difficult to draw novel data samples from the decoder model, $p_\theta(x|z)$.







\subsection{Variational Autoencoders}
Variational autoencoders (VAEs) \cite{kingma2013auto} specify a prior distribution, $p(z)$ to which $q_\phi(z|x)$ should map all $x$ samples, by formulating and maximising a variational lower bound on the log-likelihood of $p_\theta(x)$.

The variational lower bound on the log-likelihood of $p_\theta(x)$ is given by \cite{kingma2013auto}:
\begin{equation} \label{ELBO}
\log p_\theta(x) \geq \mathrm{E}_{z\sim q_\phi(z|x)}[ \log p_\theta(x|z)] - KL[q_\phi(z|x) || p(z)]
\end{equation}

The $p_\theta(x|z)$ term corresponds to the likelihood of a reconstructed $x$ given the encoding, $z$ of a data sample $x$. This formulation of the variational lower bound does not involve a corruption process. The term $KL[q_\phi(z|x) || p(z)]$ is the Kullback-Libeller divergence between $q_\phi(z|x)$ and $p(z)$. Samples are drawn from $q_\phi(z|x)$ via a re-parametrisation trick, see the VAE in Figure \ref{autoencoders}.



If $q_\phi(z|x)$ is chosen to be a parametrised multivariate Gaussian, $\mathcal{N}( \mu_\phi(x),\sigma_\phi(x))$, and the prior is chosen to be a Gaussian distribution, then $KL[q_\phi(z|x) || p(z)]$ may be computed analytically. $KL$ divergence may only be computed analytically for certain (limited) choices of prior and posterior distributions.

VAE training encourages $q_\phi(z|x)$ to map observed samples to the chosen prior, $p(z)$. Therefore, novel observed data samples may be generated via the following \textbf{simple} sampling process: $z_i \sim p(z)$, $x_i\sim p_\theta(x|z_i)$ \cite{kingma2013auto}.


Note that despite the benefits of the denoising criterion shown by Vincent et al. \cite{vincent2008extracting, vincent2010stacked}, no corruption process was introduced by Kingma et al. \cite{kingma2013auto} during VAE training.

\subsection{Denoising Variational Autoencoders}
Adding the denoising criterion to a variational autoencoder is non-trivial because the variational lower bound becomes intractable. 

Consider the conditional probability density function, $\tilde{q}_\phi(z|x)=\int q_\phi(z|\tilde{x})c(\tilde{x}|x) d\tilde{x}$, where $q_\phi(z|\tilde{x})$ is the probabilistic encoder conditioned on corrupted $x$ samples, $\tilde{x}$, and $c(\tilde{x}|x)$ is a corruption process. The variational lower bound may be formed in the following way \cite{im2017denoising}:
\[ \log p_\theta(x) \geq \mathrm{E}_{\tilde{q}_\phi(z|x)} \log \left [ \frac{p_\theta(x,z)}{q_\phi(z|\tilde{x})}\right ] \geq  \mathrm{E}_{\tilde{q}_\phi(z|x)} \log \left [ \frac{p_\theta(x,z)}{\tilde{q}_\phi(x|z)}\right ]\]

If $q_\phi(z|\tilde{x})$ is chosen to be Gaussian, then in many cases $\tilde{q}_\phi(z|x)$ will be a mixture of Gaussians. If this is the case, there is no analytical solution for $KL[\tilde{q}_\phi(z|x)|| p(z)]$ and so the denoising variational lower bound becomes analytically intractable. However there may still be an analytical solution for $KL[q_\phi(z|\tilde{x})|| p(z)]$. The denoising variational autoencoder therefore maximises $\mathrm{E}_{\tilde{q}(x|z)} \log [ \frac{p_\theta(x,z)}{q_\phi(z|\tilde{x})} ]$. We refer to the model which is trained to maximise this objective as a DVAE, see the DVAE in Figure \ref{autoencoders}. Im et al. \cite{im2017denoising} show that the DVAE achieves lower negative variational lower bounds than the regular variational autoencoder on a test dataset.

However, note that $q_\phi(z|\tilde{x})$ is matched to the prior, $p(z)$ rather than $\tilde{q}_\phi(z|x)$. This means that generating novel samples using $p_\theta(z|x)$ is not as simple as the process of generating samples from a variational autoencoder. To generate novel samples, we should sample $z_i \sim \tilde{q}_\phi(z|x)$, $x_i \sim p_\theta(x|z_i)$, which is difficult because of the need to evaluate $\tilde{q}_\phi(z|x)$. Im et al. \cite{im2017denoising} do not address this problem.

For both DVAEs and VAEs there is a limited choice of prior and posterior distributions for which there exists an analytic solution for the KL divergence. Alternatively, adversarial training may be used to learn a model that matches samples to an arbitrarily complicated target distribution -- provided that samples may be drawn from both the target and model distributions.



\section{Related Work}
\subsection{Adversarial Training}
\label{adv}
In adversarial training \cite{goodfellow2014generative} a model $g_\phi(w|v)$ is trained to produce output samples, $w$ that match a target probability distribution $t(w)$. This is achieved by iteratively training two competing models, a generative model, $g_\phi(w|v)$ and a discriminative model, $d_\chi(w)$. 
The discriminative model is fed with samples either from the generator (i.e. `fake' samples) or with samples from the target distribution (i.e. `real' samples), and trained to correctly predict whether samples are `real' or `fake'. The generative model - fed with input samples $v$, drawn from a chosen prior distribution, $p(v)$ - is trained to generate output samples $w$ that are indistinguishable from target $w$ samples in order to `fool' \cite{radford2015unsupervised} the discriminative model into making incorrect predictions. This may be achieved by the following mini-max objective \cite{goodfellow2014generative}:
\[ \min_g \max_d \mathrm{E}_{w \sim t(w)} [\log d_\chi(w)] + \mathrm{E}_{w \sim g_\phi(w|v)} [\log (1- d_\chi(w))] \]



It has been shown that for an optimal discriminative model, optimising the generative model is equivalent to minimising the Jensen-Shannon divergence between the generated and target distributions \cite{goodfellow2014generative}. In general, it is reasonable to assume that, during training, the discriminative model quickly achieves near optimal performance \cite{goodfellow2014generative}. This property is useful for learning distributions for which the Jensen-Shannon divergence may not be easily calculated.

The generative model is optimal when the distribution of generated samples matches the target distribution. Under these conditions, the discriminator is maximally confused and cannot distinguish `real' samples from `fake' ones. As a consequence of this, adversarial training may be used to capture very complicated data distributions, and has been shown to be able to synthesise images of handwritten digits and human faces that are almost indistinguishable from real data \cite{radford2015unsupervised}.
\subsection{Adversarial Autoencoders}

Makhzani et al. \cite{makhzani2015adversarial} introduce the adversarial autoencoder (AAE), where $q_\phi(z|x)$ is both the probabilistic encoding model in an autoencoder framework and the generative model in an adversarial framework.
A new, discriminative model, $d_\chi(z)$ is introduced. This discriminative model is trained to distinguish between latent samples drawn from $p(z)$ and $q_\phi(z|x)$. The cost function used to train the discriminator, $d_\chi(z)$ is:
\[ \mathcal{L}_{dis} = - \frac{1}{N}\sum_{i=0}^{N-1} [\log d_\chi (z_i)] - \frac{1}{N}\sum_{j=N}^{2N} [ \log (1- d_\chi(z_j))] \]

where $z_{i=1...N-1} \sim p(z)$ and $z_{j=N...2N} \sim q_\phi(z|x)$ and $N$ is the size of the training batch.

Adversarial training is used to match $q_\phi(z|x)$ to an arbitrarily chosen prior, $p(z)$. The cost function for matching $q_\phi(z|x)$ to prior, $p(z)$ is as follows:

\begin{equation}
\label{eqn:prior}
\mathcal{L}_{prior}= \frac{1}{N}\sum_{i=0}^{N-1} [\log(1-d_\chi(z_i))]
\end{equation}

where $z_{i=0...N-1} \sim q_\phi(z|x)$ and $N$ is the size of a training batch. 
If both $\mathcal{L}_{prior}$ and $\mathcal{L}_{dis}$ are optimised, $q_\phi(z|x)$ will be indistinguishable from $p(z)$.

In Makhzani et al.'s \cite{makhzani2015adversarial} adversarial autoencoder, $q_\phi(z|x)$ is specified by a neural network whose input is $x$ and whose output is $z$. This allows $q_\phi(z|x)$ to have arbitrary complexity, unlike the VAE where the complexity of $q_\phi(z|x)$ is usually limited to a Gaussian. In an adversarial autoencoder the posterior does not have to be analytically defined because an adversary is used to match the prior, avoiding the need to analytically compute a KL divergence. 

Makhzani et al. \cite{makhzani2015adversarial} demonstrate that adversarial autoencoders are able to match $q_\phi(z|x)$ to several different priors, $p(z)$, including a mixture of 10 2D-Gaussian distributions. We explore another direction for adversarial autoencoders, by extending them to incorporate a denoising criterion.

\section{Denoising Adversarial Autoencoder} \label{DAA}
We propose denoising adversarial autoencoders - denoising autoencoders that use adversarial training to match the distribution of auxiliary variables, $z$ to a prior distribution, $p(z)$.

We formulate two versions of a denoising adversarial autoencoder which are trained to approximately maximise the denoising variational lower bound \cite{im2017denoising}. In the first version, we directly match the posterior $\tilde{q}_\phi(z|x)$ to the prior, $p(z)$ using adversarial training. We refer to this as an integrating Denoising Adversarial Autoencoder, iDAAE. In the second, we match intermediate conditional probability distribution $q_\phi(z|\tilde{x})$ to the prior, $p(z)$. We refer to this as a DAAE.

In the iDAAE, adversarial training is used to bypass analytically intractable KL divergences \cite{im2017denoising}. In the DAAE, using adversarial training broadens the choice for prior and posterior distributions beyond those for which the KL divergence may be analytically computed.
    
\subsection{Construction} \label{construction}
The distribution of encoded data samples is given by $\tilde{q}_\phi(z|x) = \int q_\phi(z|\tilde{x})c(\tilde{x}|x)d \tilde{x}$ \cite{im2017denoising}. The distribution of decoded data samples is given by $p_\theta(x|z)$. Both $q_\phi(z|x)$ and $p_\theta(x|z)$ may be trained to maximise the likelihood of a reconstructed sample, by minimising the reconstruction cost function, $\mathcal{L}_{rec} = \frac{1}{N}\sum_{i=0}^{N-1} \log p_\theta(x|z_i)$
where the $z_i$ are obtained via the following sampling process $x_{i=0...N-1} \sim p(x)$, $\tilde{x}_i \sim c(\tilde{x}|x_i)$, $z_i \sim q_\phi(z|\tilde{x}_i)$, and $p(x)$ is distribution of the training data.

We also want to match the distribution of auxiliary variables, $z$ to a prior, $p(z)$. When doing so, there is a choice to match either $\tilde{q}_\phi(z|x)$ or $q_\phi(z|\tilde{x})$ to $p(z)$. Each choice has its own trade-offs either during training or during sampling.

\subsubsection{iDAAE: Matching $\tilde{q}_\phi(z|x)$ to a prior}
\label{v1}



In DVAEs there is often no analytical solution for the KL divergence between $\tilde{q}_\phi(z|x)$ and $p(z)$ \cite{im2017denoising}, making it difficult to match $\tilde{q}_\phi(z|x)$ to $p(z)$. Rather, we propose using adversarial training to match $\tilde{q}_\phi(z|x)$ to $p(z)$, requiring samples to be drawn from $\tilde{q}_\phi(z|x)$ during training.
It is challenging to draw samples directly from $\tilde{q}_\phi(z|x)=\int q_\phi(z|\tilde{x})c(\tilde{x}|x) d\tilde{x}$, but it is easy to draw samples from $q_\phi(z|\tilde{x})$ and so $\tilde{q}_\phi(z|x)$ may be approximated by $\frac{1}{M}\sum_{i=1}^{M} q_\phi(z|\tilde{x_i})$, $\tilde{x}_{i=1...M} \sim c(\tilde{x}|x_0)$, $x_0 \sim p(x)$ where $x \sim p(x)$ are samples from the training data, see Figure \ref{autoencoders}. Matching is achieved by minimising the following cost function:

\[ \mathcal{L}_{prior}= \frac{1}{N}\sum_{i=0}^{N-1} [\log(1-d_\chi(\hat{z}_i))]\]
where $\hat{z}_{i=0...N-1} = \frac{1}{M}\sum_{j=1}^{M} z_{i,j}$, $z_{i=1...N,j=1...M} \sim q_\phi(z|\tilde{x}_{i,j})$, $\tilde{x}_{i=1...N,j=1...M} \sim c(\tilde{x}|x_i)$, $x_{i=0...N-1} \sim p(x)$.
\label{Monte-Carlo}

\subsubsection{DAAE: Matching $q_\phi(z|\tilde{x})$ to a prior} \label{v2}
Since drawing samples from $q_\phi(z|\tilde{x})$ is trivial, $q_\phi(z|\tilde{x})$ may be matched to $p(z)$ via adversarial training. This is more efficient than matching $\tilde{q}_\phi(z|x)$ since a Monte-Carlo integration step (in Section \ref{Monte-Carlo}) is not needed, see Figure \ref{autoencoders}. In using adversarial training in place of $KL$ divergence, the only restriction is that we must be able to draw samples from the chosen prior. Matching may be achieved by minimising the following loss function:


\[ \mathcal{L}_{prior}= \frac{1}{N}\sum_{i=0}^{N-1} [\log(1-d_\chi(z_i))]\]
where $z_{i=1...N-1} \sim q_\phi(z|\tilde{x_i})$.

Though more computationally efficient to train, there are drawbacks when trying to synthesise novel samples from $p_\theta(x)$ if $q_\phi(z|\tilde{x})$ -- rather than $\tilde{q}_\phi(z|x)$ -- is matched to the prior. The effects of using a DAAE rather than an iDAAE may be visualized by plotting the empirical distribution of encodings of both data samples and corrupted data samples with the desired prior, these are shown in Figure \ref{faces:hists}.

\section{Synthesising novel samples} \label{synthesis}

In this section, we review several techniques used to draw samples from trained autoencoders, identify a problem with sampling DVAEs, which also applies to DAAEs, and propose a novel approach to sampling DAAEs; we draw strongly on previous work by Bengio et al. \cite{bengio2014deep, bengio2013generalized}.

\subsection{Drawing Samples From Autoencoders}



New samples may be generated by sampling a learned $p_\theta(x|z)$, conditioning on $z$ drawn from a suitable distribution. In the case of variational \cite{kingma2013auto} and adversarial \cite{makhzani2015adversarial} autoencoders, the choice of this distribution is simple, because during training the distribution of auxiliary variables is matched to a chosen prior distribution, $p(z)$. It is therefore easy and efficient to sample both variational and adversarial autoencoders via the following process: $z \sim p(z)$, $x \sim p_\theta(x|z)$ \cite{kingma2013auto, makhzani2015adversarial}.


The process for sampling denoising autoencoders is more complicated. In the case where the auxiliary variable is a corrupted image, $\tilde{x}$ \cite{bengio2009learning}, the sampling process is as follows: $x_0 \sim p(x)$, $\tilde{x}_0 \sim c(\tilde{x}|x_0)$, $x_1 \sim p_\theta(x|\tilde{x}_0)$  \cite{bengio2013generalized}. In the case where the auxiliary variable is an encoding, \cite{vincent2008extracting, vincent2010stacked} the sampling process is the same, with $p_\theta(x|\tilde{x})$ encompassing both the encoding and decoding process.

However, since a denoising autoencoder is trained to reconstruct corrupted versions of its inputs, $x_1$ is likely to be very similar to $x_0$. Bengio et al. \cite{bengio2013generalized} propose a method for iteratively sampling denoising autoencoders by defining a Markov chain whose stationary distribution - under certain conditions - exists and is equivalent, under certain assumption, to the training data distribution. This approach is generalised and extended by Bengio et al. \cite{bengio2014deep} to introduce a latent distribution with no prior assumptions on $z$. 

We now consider the implication for drawing samples from denoising adversarial autoencoders introduced in Section \ref{construction}. By using the iDAAE formulation (Section~ \ref{v1}) -- where $\tilde{q}_\phi(z|x)$ is matched to the prior over $z$ -- $x$ samples may be drawn from $p_\theta(x|z)$ conditioning on $z \sim p(z)$. However, if we use the DAAE -- matching $q_\phi(z|\tilde{x})$ to a prior -- sampling becomes non-trivial.

On the surface, it may appear easy to draw samples from DAAEs (Section~ \ref{v2}), by first sampling the prior, $p(z)$ and then sampling $p_\theta(x|z)$. However, the full posterior distribution is given by $\tilde{q}_\phi(z|x) = \int q_\phi(z|\tilde{x})c(\tilde{x}|x)d\tilde{x}$, but only $q_\phi(z|\tilde{x})$ is matched to $p(z)$ during training (See figure \ref{faces:hists}). The implication of this is that, when attempting to synthesize novel samples from $p_\theta(x|z)$, drawing samples from the prior, $p(z)$, is unlikely to yield samples consistent with $p(x)$. This will become more clear in Section \ref{sampling}.

\begin{figure}
\centering
\begin{subfigure}{0.45\columnwidth}
\includegraphics[width=\columnwidth]{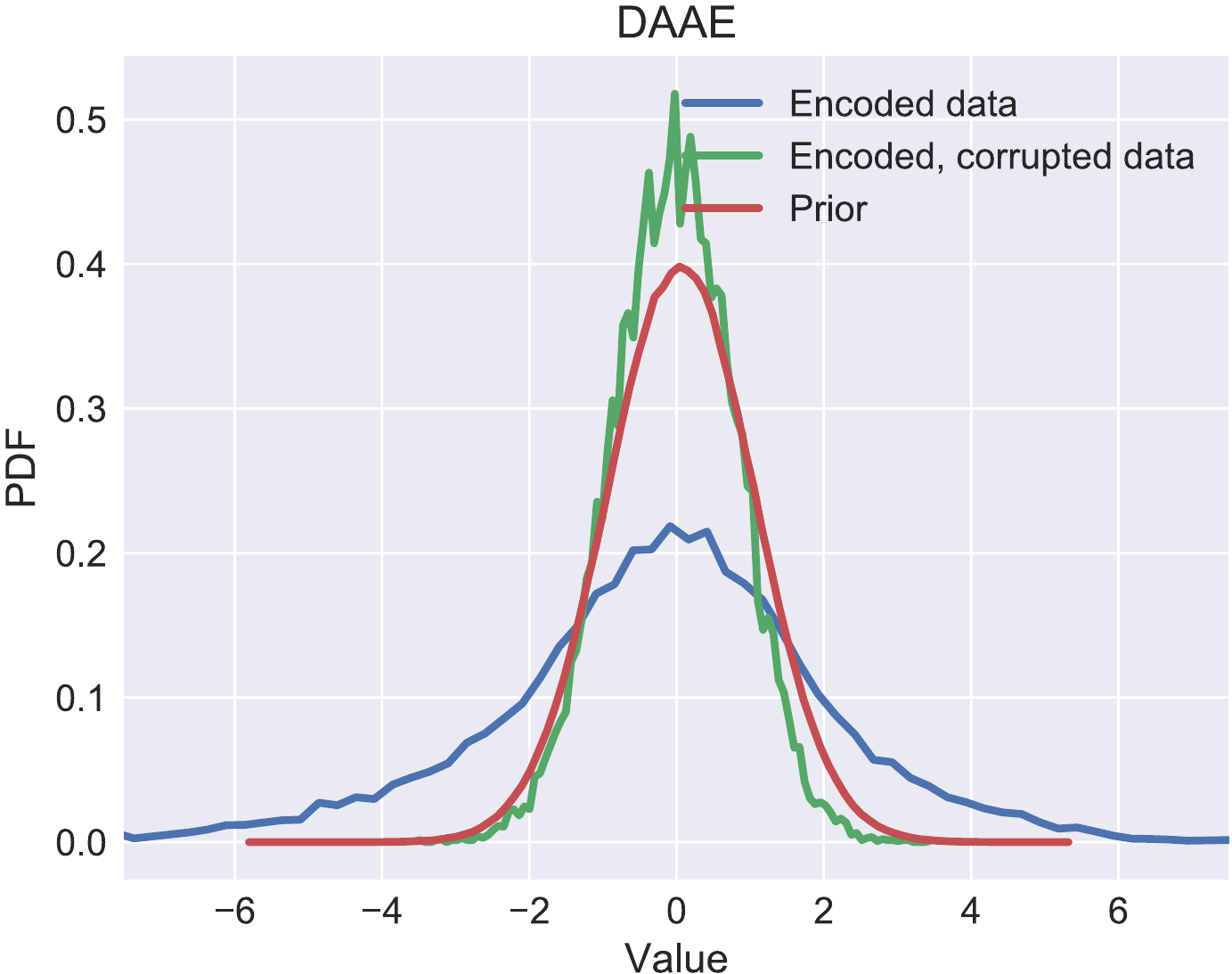}
\caption{\textbf{DAAE}}
\end{subfigure}
\begin{subfigure}{0.45\columnwidth}
\includegraphics[width=\columnwidth]{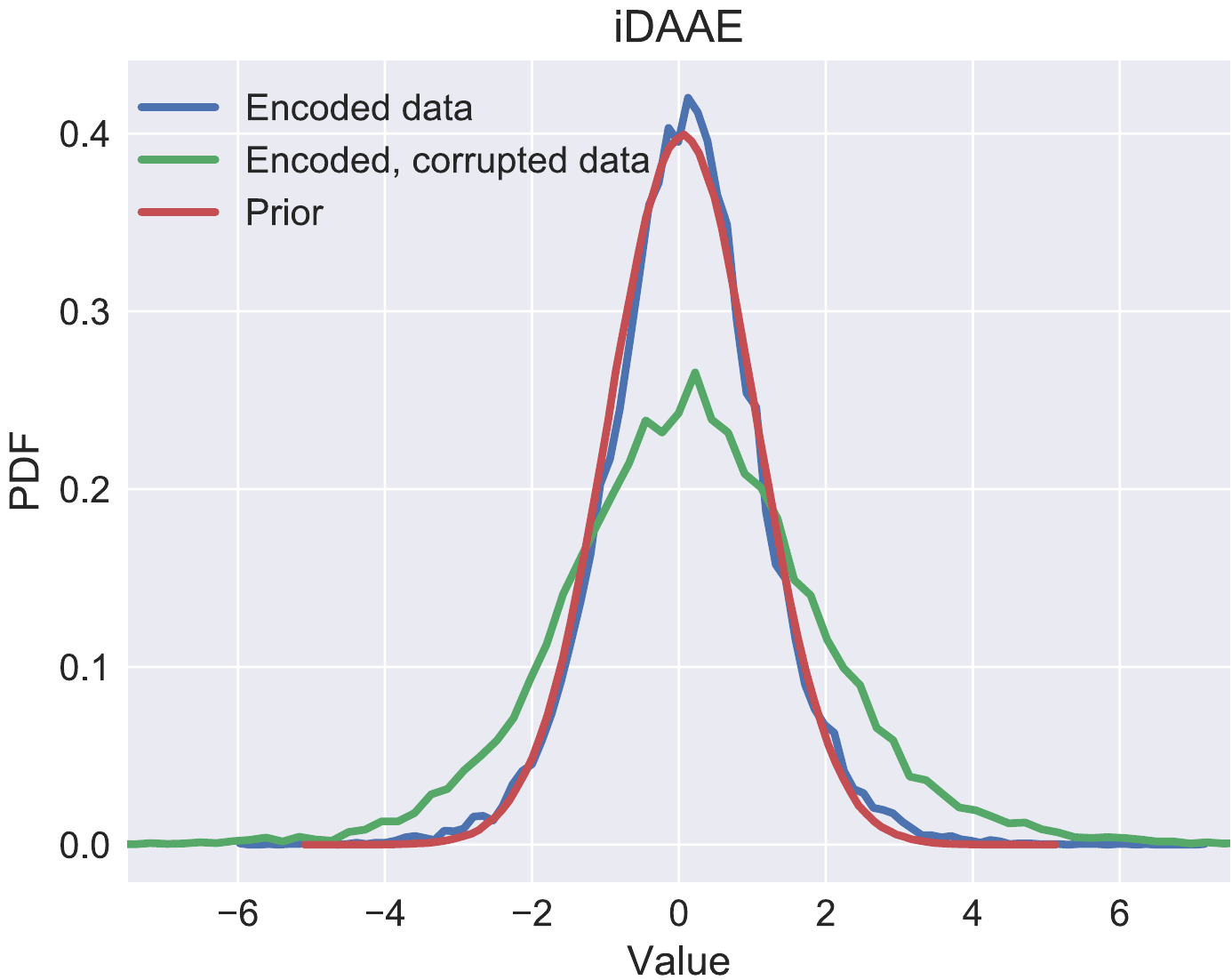}
\caption{\textbf{iDAAE}}
\end{subfigure}
\caption{ \textbf{Compare how iDAAE and DAAE match encodings to the prior when trained on the CelebA dataset}. {\color{blue} \textbf{encoding}} refers to $q_\phi(z|x)$, {\color{red} \textbf{prior}} refers to the normal prior $p(z)$, {\color{green}\textbf{encoded corrupted data}} refers to $q_\phi(z|\tilde{x})$ (a) DAAE: Encoded corrupted samples match the prior, (b) iDAAE: Encoded data samples match the prior.}
\label{faces:hists}
\end{figure}

\subsection{Proposed Method For Sampling DAAEs}
\label{sampling}


Here, we propose a method for synthesising novel samples using trained DAAEs. In order to draw samples from $p_\theta(x|z)$, we need to be able to draw samples from $\tilde{q}_\phi(z|x)$.

To ensure that we draw novel data samples, we do not want to draw samples from the training data at any point during sample synthesis. This means that we cannot use data samples from our training data to approximately draw samples from $\tilde{q}_\phi(z|x)$.

Instead, similar to Bengio et al. \cite{bengio2013generalized}, we formulate a Markov chain, which we show has the necessary properties to converge and that the chain converges to $\mathcal{P}(z)=\int \tilde{q}_\phi(z|x)p(x)dx$. 
Unlike Bengio's formulation, our chain is initialised with a random vector of the same dimensions as the latent space, rather than a sample drawn from the training set.

We define a Markov chain by the following sampling process:
\begin{equation}
\begin{split}
\label{markovChain}
z^{(0)} \sim \mathbb{R}^a, \hspace{5mm}
x^{(t)} \sim p_\theta(x|z^{(t)}),\\
\tilde{x}^{(t)} \sim c(\tilde{x}|x^{(t)}), \hspace{5mm}
z^{(t+1)} \sim q_\phi(z|\tilde{x}^{(t)}), \\
t \geq 0.
\end{split}
\end{equation}


Notice that our first sample is any real vector of dimension $a$, where $a$ is the dimension of the latent space. This Markov chain has the transition operator:
\begin{equation}
\begin{split}
\label{eqn:T}
T_{\theta,\phi}(z^{(t+1)}|z^{(t)}) =\\
\int q_\phi(z^{(t+1)}|\tilde{x}^{(t)}) c(\tilde{x}^{(t)}|x^{(t)}) p_\theta(x^{(t)}|z^{(t)}) dx d\tilde{x}
\end{split}
\end{equation}

We will now show that under certain conditions this transition operator defines an ergodic Markov chain that converges to $\mathcal{P}(z)=\int \tilde{q}_\phi(z|x) p(x) dx$ in the following steps: 1) We will show that that there exists a stationary distribution $\mathcal{P}(z)$ for $z^{(0)}$ drawn from a specific choice of initial distribution (Lemma \ref{p(x)_main}). 2) The Markov chain is homogeneous, because the transition operator is defined by a set of distributions whose parameters are fixed during sampling. 3) We will show that the Markov chain is also ergodic, (Lemma \ref{ergodic_main}). 4) Since the chain is both homogeneous and ergodic there exists a unique stationary distribution to which the Markov chain will converge \cite{textbook}. 

Step 1) shows that one stationary distribution is $\mathcal{P}(z)$, which we now know by 2) and 3) to be the unique stationary distribution. So the Markov chain converges to $\mathcal{P}(z)$.


In this section, only, we use a change of notation, where the training data probability distribution, previously represented as $p(x)$ is represented as $\mathcal{P}(x)$, this is to help make distinctions between ``natural system'' probability distributions and the learned distributions. Further, note that $p(z)$ is the prior, while the distribution required for sampling $\mathcal{P}(x|z)$ is $\mathcal{P}(z)$ such that:
\begin{equation} \label{assumption1}
\mathcal{P}(x) = \int \mathcal{P}(x|z)\mathcal{P}(z) dz \approx \int p_\theta(x|z)\mathcal{P}(z) dz.
\end{equation}
\begin{equation} \label{pz}
\mathcal{P}(z)=\int \tilde{q}_\phi(z|x) \mathcal{P}(x) dx = \int \int q_\phi(z|\tilde{x})c(\tilde{x}|x) d\tilde{x} \mathcal{P}(x) dx.
\end{equation}




\begin{lemma} \label{p(x)_main} $\mathcal{P}(z)$ is a stationary distribution for the Markov chain defined by the sampling process in (\ref{markovChain}). 
\end{lemma}

For proof see Appendix.

\begin{lemma} \label{ergodic_main} The Markov chain defined by the transition operator, $T_{\theta,\phi}(z_{t+1}|z_t)$ (\ref{eqn:T}) is ergodic, provided that the corruption process is additive Gaussian noise and that the adversarial pair, $q_\phi(z|\tilde{x})$ and $d_\chi(z)$ are optimal within the adversarial framework.
\end{lemma}

For proof see Appendix.

\begin{theorem} \label{stationary_mainß} Assuming that $p_\theta(x|z)$ is approximately equal to $\mathcal{P}(x|z)$, and that the adversarial pair -- $q_\phi(z|x)$ and $d_\chi(z)$ -- are optimal, the transition operator $T_{\theta,\phi}(z^{(t+1)}|z^{(t)})$ defines a Markov chain whose unique stationary distribution is $\mathcal{P}(z)=\int \tilde{q}_\phi(z|x)\mathcal{P}(x)dx$.
\end{theorem}

\begin{proof}
This follows from Lemmas \ref{p(x)_main} and \ref{ergodic_main}.
\end{proof}

This sampling method uncovers the distribution $\mathcal{P}(z)$ on which samples drawn from $p_\theta(x|z)$ must be conditioned in order to sample $p_\theta(x)$. Assuming $p_\theta(x|z) \approx \mathcal{P}(x|z)$, this allows us to draw samples from $\mathcal{P}(x)$.

For completeness, we would like to acknowledge that there are several other methods that use Markov chains during the training of autoencoders \cite{bachman2015variational, nguyen2016plug} to improve performance. Our approach for synthesising samples using the DAAE is focused on sampling only from trained models; the Markov chain sampling is not used to update model parameters.


\section{Implementation}
The analyses of Sections \ref{DAA} and \ref{synthesis} are deliberately general: they do not rely on any specific implementation choice to capture the model distributions.  In this section, we consider a specific implementation of denoising adversarial autoencoders and apply them to the task of learning models for image distributions. We define an encoding model that maps corrupted data samples to a latent space $E_\phi(\tilde{x})$, and $R_\theta(z)$ which maps samples from a latent space to an image space. These respectively draw samples according to the conditional probabilities $q_\phi(z|\tilde{x})$ and $p_\theta(x|z)$. We also define a corruption process, $C(x)$, which draws samples according to $c(\tilde{x}|x)$.


The parameters $\theta$ and $\phi$ of models $R_\theta(z)$ and $E_\phi(z)$ are learned under an autoencoder framework; the parameters $\phi$ are also updated under an adversarial framework. The models are trained using large datasets of unlabelled images.

\subsection{The Autoencoder}
Under the autoencoder framework, $E_\phi(x)$ is the encoder and $R_\theta(z)$ is the decoder. We used fully connected neural networks for both the encoder and decoder. Rectifying Linear Units (ReLU) were used between all intermediate layers to encourage the networks to learn representations that capture multi-modal distributions. In the final layer of the decoder network, a sigmoid activation function is used so that the output represents pixels of an image. The final layer of the encoder network is left as a linear layer, so that the distribution of encoded samples is not restricted.

As described in Section \ref{construction}, the autoencoder is trained to maximise the log-likelihood of the reconstructed image given the corrupted image. Although there are several ways in which one may evaluate this log-likelihood, we chose to measure pixel wise binary cross-entropy between the reconstructed sample, $\hat{x}$ and the original samples before corruption, $x$. During training we aim to learn parameters $\phi$ and $\theta$ that minimise the binary cross-entropy between $\hat{x}$ and $x$. 
The training process is summarised by lines $1$ to $9$ in Algorithm \ref{algv1} in the Appendix.

The vectors output by the encoder may take any real values, therefore minimising reconstruction error is not sufficient to match either $q_\phi(z|\tilde{x})$ or $\tilde{q}_\phi(z|x)$ to the prior, $p(z)$. For this, parameters $\phi$ must also be updated under the adversarial framework.

\subsection{Adversarial Training}
To perform adversarial training we define the discriminator $d_\chi(z)$, described in Section \ref{adv} to be a fully connected neural network, which we denote $D_\chi(z)$. The output of $D_\chi(z)$ is a ``probability'' because the final layer of the neural network has a sigmoid activation function, constraining the range of $D_\chi(z)$ to be between $(0,1)$. Intermediate layers of the network have ReLU activation functions to encourage the network to capture highly non-linear relations between $z$ and the labels, \{`real', `fake'\}.

How adversarial training is applied depends on whether $\tilde{q}_\phi(z|x)$ or $q_\phi(z|\tilde{x})$ is being fit to the prior $p(z)$. $z_{fake}$ refers to the samples drawn from the distribution that we wish to fit to $p(z)$ and $z_{real}$, samples drawn from the prior, $p(z)$. The discriminator, $D_\chi(z)$, is trained to predict whether $z$'s are `real' or `fake'. This may be achieved by learning parameters $\chi$ that maximise the probability of the correct labels being assigned to $z_{fake}$ and $z_{real}$. This training procedure is shown in Algorithm \ref{algv1} on Lines $14 $ to $16$.

Drawing samples, $z_{real}$, involves sampling some prior distribution, $p(z)$, often a Gaussian. Now, we consider how to draw fake samples, $z_{fake}$. How these samples are drawn depends on whether $q_\phi(z|\tilde{x})$ (DAAE) is being fit to the prior or $\tilde{q}_\phi(z|x)$ (iDAAE) is being fit to the prior. Drawing samples, $z_{fake}$ is easy if $q_\phi(z|\tilde{x})$ is being matched to the prior, as these are simply obtained by mapping corrupted samples though the encoder: $z_{fake} = E_\phi(\tilde{x})$.

However, if $\tilde{q}(z|x)$ is being matched to the prior, we must use Monte Carlo sampling to approximate $z_{fake}$ samples (see Section \ref{v1}). The process for calculating $z_{fake}$ is given by Algorithm \ref{algz} in the Appendix, and detailed in Section \ref{v1}.

Finally, in order to match the distribution of $z_{fake}$ samples to the prior, $p(z)$, adversarial training is used to update parameters $\phi$ while holding parameters $\chi$ fixed. Parameters $\phi$ are updated to minimise the likelihood that $D_\chi(\cdot)$ correctly classifies $z_{fake}$ as being `fake'. The training procedure is laid out in lines $18$ and $19$ of Algorithm \ref{algv1}.

Algorithm \ref{algv1} shows the steps taken to train an iDAAE. To train a DAAE
instead, all lines in Algorithm \ref{algv1} are the same except Line $11$, which may be replaced by $z_{fake} = E_\phi(\tilde{x})$.

\subsection{Sampling}
Although the training process for matching $\tilde{q}_\phi(z|x)$ to $p(z)$ is less computationally efficient than matching $q_\phi(z|\tilde{x})$ to $p(z)$, it is very easy to draw samples when $\tilde{q}_\phi(z|x)$ is matched to the prior (iDAAE). We simply draw a random $z^{(0)}$ value from $p(z)$, and calculate $x^{(0)}=R_\theta(z^{(0)})$, where $x^{(0)}$ is a new sample. When drawing samples, parameters $\theta$ and $\phi$ are fixed.

If $q_\phi(z|\tilde{x})$ is matched to the prior (DAAE), an iterative sampling process is needed in order to draw new samples from $p(x)$. This sampling process is described in Section \ref{sampling}. To implement this sampling process is trivial. A random sample, $z^{(0)}$ is drawn from any distribution; the distribution does not have to be the chosen prior, $p(z)$. New samples, $z^{(t)}$ are obtained by iteratively decoding, corrupting and encoding $z^{(t)}$, such that $z^{(t+1)}$ is given by:
\[ z^{(t+1)} = E_\phi(C(R_\theta(z^{(t)}))). \]

In the following section, we evaluate the performance of denoising adversarial autoencoders on three image datasets, a handwritten digit dataset (MNIST) \cite{lecun1998mnist}, a synthetic colour image dataset of tiny images (Sprites) \cite{reed2015deep}, and a complex dataset of hand-written characters \cite{lake2015human}. The denoising and non-denoising adversarial autoencoders (AAEs) are compared for tasks including reconstruction, generation and classification. 

\section{Experiments \& Results}

\subsection{Code Available Online}
We make our PyTorch \cite{paszke2017automatic} code available at the following link: \url{https://github.com/ToniCreswell/pyTorch_DAAE}  \footnote{An older version of our code in Theano available at \url{https://github.com/ToniCreswell/DAAE_} with our results presented in iPython notebooks. Since this is a revised version of our paper and Theano is no longer being supported, our new experiments on the CelebA datasets were performed using PyTorch.}.


\subsection{Datasets}
We evaluate our denoising adversarial autoencoder on three image datasets of varying complexity. Here, we describe the datasets and their complexity in terms of variation within the dataset, number of training examples and size of the images.

\subsubsection{Datasets: Omniglot}
\label{sec:data:omniglot}
The Omniglot dataset is a handwritten character dataset consisting of $1623$ categories of character from $50$ different writing systems, with only $20$ examples of each character. Each example in the dataset is $105$-by-$105$ pixels, taking values \{0,1\}. The dataset is split such that $19$ examples from $964$ categories make up the training dataset, while one example from each of those $964$ categories makes up the testing dataset. The $20$ characters from each of the remaining $659$ categories make up the evaluation dataset. This means that experiments may be performed to reconstruct or classify samples from categories not seen during training of the autoencoders.
\subsubsection{Datasets: Sprites} The sprites dataset is made up of $672$ unique human-like characters. Each character has $7$ attributes including hair, body, armour, trousers, arm and weapon type, as well as gender. For each character there $20$ animations consisting of $6$ to $13$ frames each. There are between $120$ and $260$ examples of each character, however every example is in a different pose. Each sample is $60$-by-$60$ pixels and samples are in colour.
The training, validation and test datasets are split by character to have $500$, $72$ and $100$ unique characters each, with no two sets having the same character.

\subsubsection{Datasets: CelebA} The CelebA dataset consists of $250$k images of faces in colour. Though a version of the dataset with tightly cropped faces exists, we use the un-cropped dataset. We use $1000$ samples for testing and the rest for training. Each example has dimensions $64$-by-$64$ and a set of labelled facial attributes for example, `No Beard', `Blond Hair', `Wavy Hair' etc. . This face dataset is more complex than the Toronto Face Dataset used by Makhzani et al. \cite{makhzani2015adversarial} for training the AAE.  

\subsection{Architecture and Training}
For each dataset, we detail the architecture and training parameters of the networks used to train each of the denoising adversarial autoencoders. For each dataset, several DAAEs, iDAAEs and AAEs are trained. In order to compare models trained on the same datasets, the same network architectures, batch size, learning rate, annealing rate and size of latent code is used for each.
Each set of models were trained using the same optimization algorithm. The trained AAE \cite{makhzani2015adversarial} models act as a benchmark, allowing us to compare our proposed DAAEs and iDAAEs.

\subsubsection{Architecture and Training: Omniglot}
The decoder, encoder and discriminator networks consisted of $6$, $3$ and $2$ fully connected layers respectively, each layer having $1000$ neurons. We found that deeper networks than those proposed by Makhazni et al. \cite{makhzani2015adversarial} (for the MNIST dataset) led to better convergence. The networks are trained for $1000$ epochs, using a learning rate $10^{-5}$, a batch size of 64 and the Adam \cite{kingma2014adam} optimization algorithm. We used a $200$D Gaussian for the prior and additive Gaussian noise with standard deviation $0.5$ for the corruption process. When training the iDAAE, we use $M=5$ steps of Monte Carlo integration (see Algorithm \ref{algz} in the Appendix).

\subsubsection{Architecture and Training: Sprites}
Both the encoder and discriminator are $2$-layer fully connected neural networks with $1000$ neurons in each layer. For the decoder, we used a $3$-layer fully connected network with $1000$ neurons in the first layer and $500$ in each of the last layers, this configuration allowed us to capture complexity in the data without over fitting. The networks were trained for $5$ epochs, using a batch size of $128$, a learning rate of $10^{-4}$ and the Adam \cite{kingma2014adam} optimization algorithm. We used an encoding $200$ units, $200$D Gaussian for the prior and additive Gaussian noise with standard deviation $0.25$ for the corruption process. The iDAAE was trained with $M=5$ steps of Monte Carlo integration.

\subsubsection{Architecture and Training: CelebA}
The encoder and decoder were constructed with convolutional layers, rather than fully connected layers since the CelebA dataset is more complex than the Toronto face dataset use by Makhzani et al. \cite{makhzani2015adversarial}. The encoder and decoder consisted of $4$ convolutional layers with a similar structure to that of the DCGAN proposed by Radford et al. \cite{radford2015unsupervised}. We used a $3$-layer fully connected network for the discriminator. Networks were trained for $100$ epochs with a batch size of $64$ using RMSprop with learning rate $10^{-4}$ and momentum of $\rho=0.1$ for training the discriminator. We found that using smaller momentum values lead to more blurred images, however larger momentum values prevented the network from converging and made training unstable. When using Adam instead of RMSprop (on the CelebA dataset specifically) we found that the values in the encodings became very large, and were not consistent with the prior. The encoding was made up of $200$ units and we used a $200$D Gaussian for the prior. We used additive Gaussian noise for the corruption process. We experimented with different noise level, $\sigma$ between $[0.1, 1.0]$, we found several values in this range to be suitable. For our classification experiments we fixed $\sigma=0.25$ and for synthesis from the DAAE, to demonstrate the effect of sampling, we used $\sigma=1.0$. For the iDAAE we experimented with $M=2, 5, 20, 50$. We found that $M < 5$ (when $\sigma=1.0$), was not sufficient to train an iDAAE. By comparing histograms of encoded data samples to histograms of the prior (see Figure \ref{faces:hists}), for an iDAAE trained with a particular $M$ value, we are able to see whether $M$ is sufficiently larger or not. We found $M=5$ to be sufficiently large for most experiments.


\subsection{Sampling DAAEs and iDAAEs}



Samples may be synthesized using the decoder of a trained iDAAE or AAE by passing latent samples drawn from the prior through the decoder. On the other hand, if we pass samples from the prior through the decoder of a trained DAAE, the samples are likely to be inconsistent with the training data. To synthesize more consistent samples using the DAAE, we draw an initial $z^{(0)}$ from any random distribution -- we use a Gaussian distribution for simplicity\footnote{which happens to be equivalent to our choice of prior} -- and decode, corrupt and encode the sample several times for each synthesized sample. This process is equivalent to sampling a Markov chain where one iteration of the Markov chain includes decoding, corrupting and encoding to get a $z^{(t)}$ after $t$ iterations.  The $z^{(t)}$ may be used to synthesize a novel sample which we call, $x^{(t)}$. $x^{(0)}$ is the sample generated when $z^{(0)}$ is passed through the decoder.

To evaluate the quality of some synthesized samples, we calculated the log-likelihood of real samples under the model \cite{makhzani2015adversarial}.  This is achieved by fitting a Parzen window to a number of synthesised samples. Further details of how the log-likelihood is calculated for each dataset is given in the Appendix \ref{sec:ll}. 


We expect initial samples, $x^{(0)}$'s drawn from the DAAE to have a lower (worse) log-likelihood than those drawn from the AAE, however we expect Markov chain (MC) sampling to improve synthesized samples, such that $x^{(t)}$ for $t > 0$ should have larger log-likelihood than the initial samples. It is not clear whether $x^{(t)}$ for $t > 0$ drawn using a DAAE will be better than samples drawn form an iDAAE. The purpose of these experiments is to demonstrate the challenges associated with drawing samples from denoising adversarial autoencoders, and show that our proposed methods for sampling a DAAE and training iDAAEs allows us to address these challenges. We also hope to show that iDAAE and DAAE samples are competitive with those drawn from an AAE.




\subsubsection{Sampling: Omniglot}
Here, we explore the Omniglot dataset, where we look at log-likelihood score on both a testing and evaluation dataset. Recall (Section \ref{sec:data:omniglot}) that the testing dataset has samples from the same classes as the training dataset and the evaluation dataset has samples from different classes.



First, we discuss the results on the evaluation dataset. The results, shown in Figure \ref{omni_ll}, are consistent with what is expected of the models. The iDAAE out-performed the AAE, with a less negative (better) log-likelihood. The initial samples drawn using the DAAE had more negative (worse) log-likelihood values than samples drawn using the AAE. However, after one iteration of MC sampling, the synthesized samples have less negative (better) log-likelihood values than those from the AAE. Additional iterations of MC sampling led to worse results, possibly because synthesized samples tending towards multiple modes of the data generating distribution, appearing to be more like samples from classes represented in the training data.

The Omniglot testing dataset consists of one example of every category in the training dataset. This means that if multiple iterations of MC sampling cause synthesized samples to tend towards modes in the training data, the likelihood score on the testing dataset is likely to increase. The results shown in Figure \ref{omni_ll} confirm this expectation; the log-likelihood for the $5^{th}$ sample is less negative (better) than for the $1^{st}$ sample. These apparently conflicting results (in Figure \ref{omni_ll}) -- whether sampling improves or worsens synthesized samples -- highlights the challenges involved with evaluating generative models using the log-likelihood, discussed in more depth by Theis et al. \cite{theis2015note}. For this reason, we also show qualitative results.


Figure \ref{omni_chain_sample}(a) shows an set of initial samples ($x^{(0)}$) drawn from a DAAE and samples synthesised after $9$ iterations ($x^{(9)}$) of MC sampling in Figure \ref{omni_chain_sample}(b), these samples appear to be well vaired, capturing multiple modes of the data generating distribution.  

\begin{figure}
    \centering
    \includegraphics[width=0.9\columnwidth]{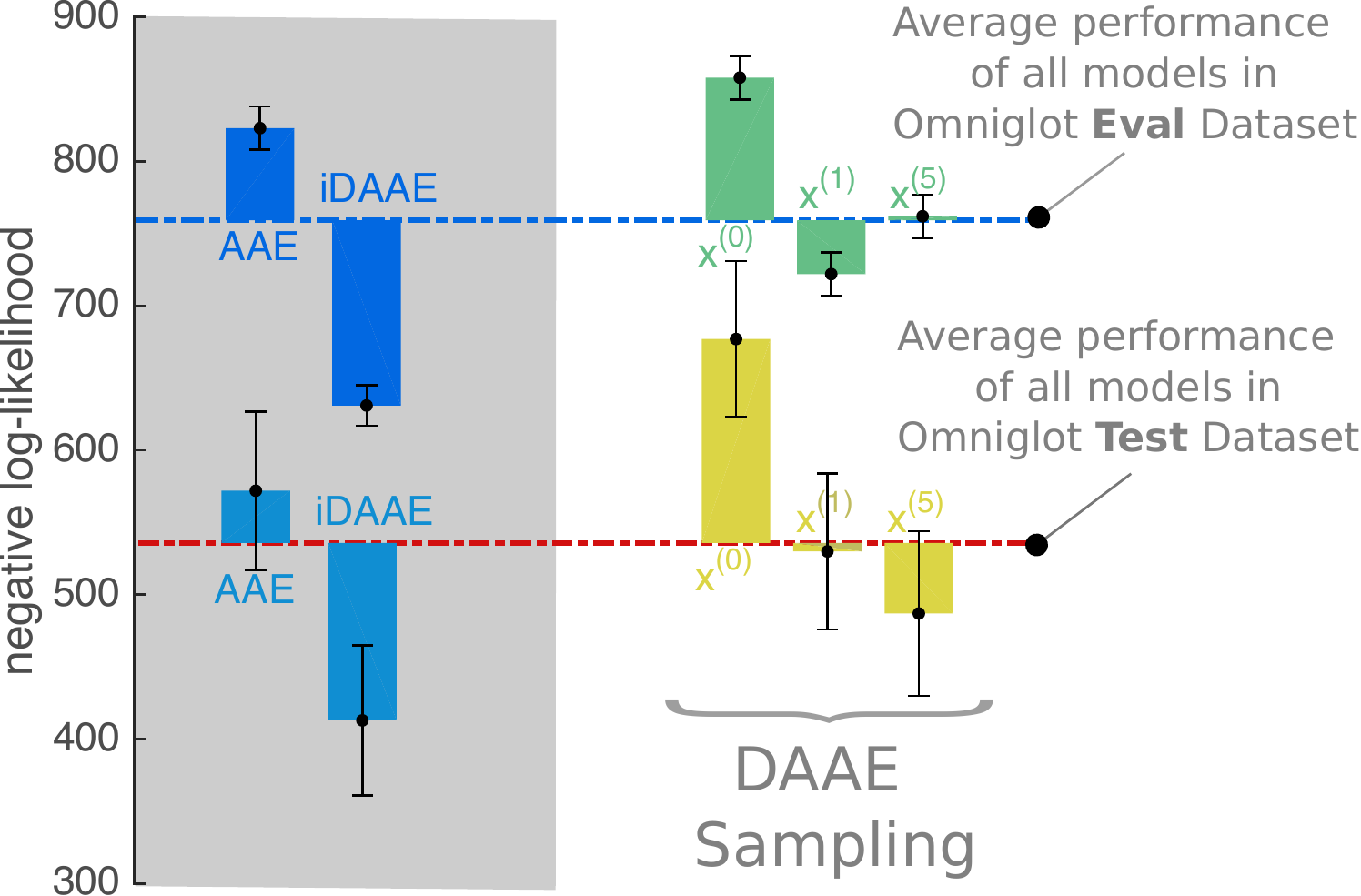}
    \caption{\textbf{Omniglot log-likelihood of $p_\theta(x)$ compared on the testing and evaluation datasets.} The training and evaluation datasets have samples from \textbf{different} handwritten character classes. All models were trained using a 200D Gaussian prior. The training and testing datasets have samples from the \textbf{same} handwritten character classes.}
    \label{omni_ll}
\end{figure}

\begin{figure}
\centering
\begin{subfigure}{0.45\columnwidth}
\includegraphics[width=0.9\columnwidth]{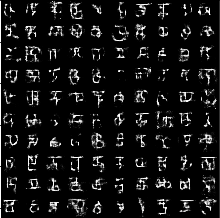}
\caption{$x^{(0)}$}
\end{subfigure}
\begin{subfigure}{0.45\columnwidth}
\includegraphics[width=0.9\columnwidth]{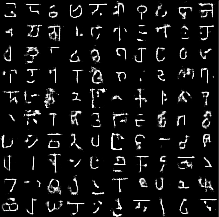}
\caption{$x^{(9)}$}
\end{subfigure}
\caption{\textbf{Omniglot Markov chain (MC) sampling:} (a) Initial sample, $x^{(0)}$ and (b) \textit{Corresponding} samples, $x^{(9)}$ after $9$ iterations of MC sampling. The chain was initialized with $z^{(0)} \sim \mathcal{N}(0,I)$.}
\label{omni_chain_sample}
\end{figure}

\subsubsection{Sampling: Sprites}
\label{sampling_sprites}

In alignment with expectation, the iDAAE model synthesizes samples with higher (better) log-likelihood, $\mathbf{2122 \pm 5}$, than the AAE, $\mathbf{2085 \pm 5}$. The initial image samples drawn from the DAAE model under-perform compared to the AAE model, $\mathbf{2056\pm 5}$, however after just one iteration of sampling the synthesized samples have higher log-likelihood than samples from the AAE. Results also show that synthesized samples drawn using the DAAE after one iteration of MC sampling have higher likelihood, $\mathbf{2261 \pm 5}$, than samples drawn using either the the iDAAE or AAE models.

When more than one step of MC sampling is applied, the log-likelihood decreases, in a similar way to results on the Omniglot evalutation dataset, this may be related to how the training, and test data are split, each dataset has a unique set of characters, so combinations seen during training will not be present in the testing dataset. These results further suggest that MC sampling pushes synthesized samples towards modes in the training data.


\subsubsection{Sampling: CelebA} 

In Figure \ref{iDAAE:faces:sampling} we compare samples synthesized using an AAE to those synthesized using an iDAAE trained using $M=5$ intergration steps. Figure \ref{iDAAE:faces:sampling}(b) show samples drawn from the iDAAE which improve upon those drawn from the AAE model. 

 In Figure \ref{DAAE:faces:sampling} we show samples synthesized from a DAAE using the iterative approach described in section \ref{sampling} for $M=\{0,5,20\}$. We see that the initial samples, $x^{(0)}$ have blurry artifacts, while the final samples, $x^{(20)}$ are more sharp and free from the blurry artifacts.

\begin{figure}[h]
\begin{subfigure}{0.48\columnwidth}
\includegraphics[width=0.9\columnwidth]{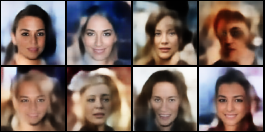} 
\caption{AAE (Previous work)}
\end{subfigure}
\begin{subfigure}{0.48\columnwidth}
\includegraphics[width=0.9\columnwidth]{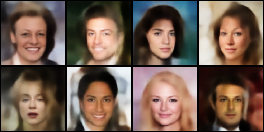}
\caption{iDAAE  (Our work)}
\end{subfigure}
\caption{\textbf{CelebA iDAAE samples} (a) AAE (no noise) with $\rho=0.1$ and (c) iDAAE (with noise) with $\rho=0.1$, $\sigma=0.25$ and $M=5$. }
\label{iDAAE:faces:sampling}
\end{figure}

\begin{figure}
\centering
\begin{subfigure}{0.3\columnwidth}
\includegraphics[width=0.8\columnwidth]{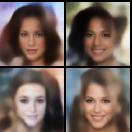}
\caption{$x^{(0)}$}
\end{subfigure}
\begin{subfigure}{0.3\columnwidth}
\includegraphics[width=0.8\columnwidth]{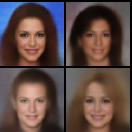}
\caption{$x^{(5)}$}
\end{subfigure}
\begin{subfigure}{0.3\columnwidth}
\includegraphics[width=0.8\columnwidth]{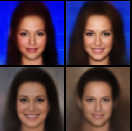}
\caption{$x^{(20)}$}
\end{subfigure}
\caption{\textbf{DAAE Face Samples} $\sigma=1.0$}
\label{DAAE:faces:sampling}
\end{figure}

When drawing samples from iDAAE and DAAE models trained on CelebA, a critical difference between the two models emerges: samples synthesized using a DAAE have good structure but appear to be quite similar to each other, while the iDAAE samples have less good structure but appear to have lots of variation. The lack of variation in DAAE samples may be related to the sampling procedure, which accoriding to theory presented by Alain et al. \cite{alain2014regularized},  would be similar to taking steps toward the highest density regions of the distribution (i.e. the mode), explaining why samples appear to be quite similar.

When comparing DAAE or iDAAE samples to samples from other generative models such as GANs \cite{goodfellow2014generative} we may notice that samples are less sharp. However, GANs often suffer from `mode collapse' this is where all synthesised samples are very similar, the iDAAE does not suffer mode collapse and does not require any additional procedures to prevent mode collapse \cite{salimans2016improved}. Further, (vanilla) GANs do not offer an encoding model. Other GAN variants such as Bi-GAN \cite{donahue2016adversarial} and ALI \cite{dumoulin2016adversarially} do offer encoding models, however the fidelity of reconstruction is very poor. The AAE, DAAE and iDAAE models are able to reconstruct samples faithfully. We will explore fidelity of reconstruction in the next section and compare to a state-of-art ALI that has been modified to have improved reconstruction fidelity, ALICE \cite{li2017alice}.



We conclude this section on sampling, by making the following observations; samples synthesised using iDAAEs out-performed AAEs on all datasets, where $M = 5$. It is convenient, that relatively small $M$ yields improvement, as the time needed to train an iDAAE may increase linearly with $M$. We also observed that initial samples synthesized using the DAAE are poor and in all cases even just one iteration of MC sampling improves image synthesis. 

Finally, evaluating generated samples is challenging: log-likelihood is not always reliable \cite{theis2015note}, and qualitative analysis is subjective. For this reason, we provided both quantitative and qualitative results to communicate the benefits of introducing MC sampling for a trained DAAE, and the advantages of iDAAEs over AAEs. 

\subsection{Reconstruction}
The reconstruction task involves passing samples from the test dataset through the trained encoder and decoder to recover a sample similar to the original (uncorrupted) sample. The reconstruction is evaluated by computing the mean squared error between the reconstruction and the original sample.

We are interested in reconstruction for several reasons. The first is that if we wish to use encodings for down stream tasks, for example classification, a good indication of whether the encoding is modeling the sample well is to check the reconstructions. For example if the reconstructed image is missing certain features that were present in the original images, it may be that this information is not preserved in the encoding. The second reason is that checking sample reconstructions is also a method to evaluate whether the model has overfit to test samples. The ability to reconstruct samples not seen during training suggests that a model has not overfit. The final reason, is to further motivate AAE, DAAE and iDAAE models as alternatives to GAN based models that are augmented with encoders \cite{li2017alice}, for down stream tasks that require good sample reconstruction.
We expect that adding noise during training would both prevent over fitting and encourage the model to learn more robust representations, therefore we expect that the DAAE and iDAAE would outperform the AAE. 

\subsubsection{Reconstruction: Omniglot}
Table \ref{table:rec} compares reconstruction errors of the AAE, DAAE and iDAAE trained on the Omniglot dataset. The reconstruction errors for both the iDAAE and the DAAE are less than the AAE. The results suggest that using the denoising criterion during training helps the network learn more robust features compared to the non-denoising variant. The smallest reconstruction error was achieved by the DAAE rather than the iDAAE; qualitatively, the reconstructions using the DAAE captured small details while the iDAAE lost some. This is likely to be related to the multimodal nature of $\tilde{q}_\phi(z|x)$ in the DAAE compared to the unimodal nature of $\tilde{q}_\phi(z|x)$ in an iDAAE.


\subsubsection{Reconstruction: Sprites}
Table \ref{table:rec} shows reconstruction error on samples from the sprite test dataset for models trained on the sprite training data. In this case only the iDAAE model out-performed the AAE and the DAAE performed as well as the AAE. 


\subsubsection{Reconstruction: CelebA}

Table \ref{table:rec} shows reconstruction error on the CelebA dataset. We compare AAE, DAAE and iDAAE models trained with momentum, $\rho=0.1$, where the DAAE and iDAAE have corruption $\sigma=0.1$ and the iDAAE is trained with $M=10$ integration steps. We also experimented with $M=5$ however, better results were obtained using $M=10$. While the DAAE performs similarly well to the AAE, the iDAAE outperforms both. Figure \ref{face:rec} shows examples of reconstructions obtained using the iDAAE. Although the reconstructions are slightly blurred, the reconstructions are highly faithful, suggesting that facial attributes are correctly encoded by the iDAAE model.

\begin{figure}
\centering
\begin{subfigure}{0.45\columnwidth}
\includegraphics[width=0.9\columnwidth]{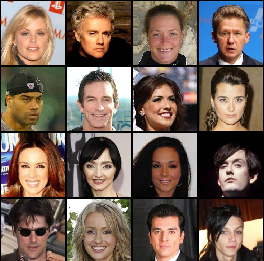} %
\caption{Original}
\end{subfigure}
\begin{subfigure}{0.45\columnwidth}
\includegraphics[width=0.9\columnwidth]{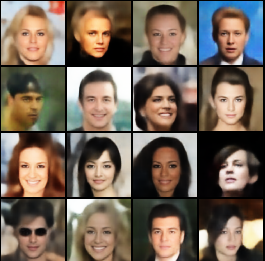} 
\caption{Reconstructions}
\end{subfigure}
\caption{ \textbf{CelebA Reconstruction Error with an iDAAE}}
\label{face:rec}
\end{figure}



\begin{table} 
\centering
\caption{\textbf{Reconstruction:} Shows the mean squared error for reconstructions of corrupted test data samples. This table server two purposes: (1) To demonstrate that in most cases the DAAE and iDAAE are better able to reconstruct images compared to the AAE. (2) To motivate why we are interested in AAEs, as opposed to other GAN \cite{goodfellow2014generative} related approaches, by comparing reconstruction error on MNIST for a state of the art GAN variant the ALICE \cite{li2017alice}, which was designed to improve reconstruction fidelity in GAN-like models.}
\begin{tabular}{lcccc}\toprule
    Model &     Omniglot &        Sprite &          MNIST\footnote{The MNIST dataset is described in the Appendix.}  &              CelebA\\ \midrule
    AAE &       0.047 &           0.019 &           0.017 &      0.500\\
    DAAE &      \textbf{0.029} &  0.019 &  \textbf{0.015} &               0.501 \\
    iDAAE &     0.031 &           \textbf{0.018} &           0.018 &               \textbf{0.495}\\
    ALICE {\cite{li2017alice}} & - & - &         0.080& \\
    \bottomrule
\end{tabular}
\label{table:rec}
\end{table}










\subsection{Classification}

We are motivated to understand the properties of the representations (latent encoding) learned by the DAAE and iDAAE trained on unlabeled data. A particular property of interest is the separability, in latent space, between objects of different classes. To evaluate separability, rather than training in a semi-supervised fashion \cite{makhzani2015adversarial} we obtain class predictions by training an SVM on top of the representations, in a similar fashion to that of Kumar et al. \cite{kumar2017variational}.

\subsubsection{Classification: Omniglot}
Classifying samples in the Omniglot dataset is very challenging: the training and testing datasets consists of $946$ classes, with only $19$ examples of each class in the training dataset. The $946$ classes make up $30$ writing systems, where symbols between writing systems may be visually indistinguishable. Previous work has focused on only classifying $5$, $15$ or $20$ classes from within a single writing system \cite{santoro2016one,vinyals2016matching,edwards2016towards,lake2015human}, however we attempt to perform classification across all $946$ classes. The Omniglot training dataset is used to train SVMs (with RBF kernels) on encodings extracted from encoding models of the trained DAAE, iDAAE and AAE models. Classification scores are reported on the Omniglot evaluation dataset, (Table \ref{omni_class}).


Results show that the DAAE and iDAAE out-perform the AAE on the classification task. The DAAE and iDAAE also out-perform a classifier trained on encodings obtained by applying PCA to the image samples, while the AAE does not, further showing the benefits of using denoising.




We perform a separate classification task using only $20$ classes from the Omniglot evaluation dataset (each class has $20$ examples). This second test is performed for two key reasons: a) to study how well autoencoders trained on only a subset of classes can generalise as feature extractors for classifiers of classes not seen during {\em autoencoder} training; b) to facilitate performance comparisons with previous work \cite{edwards2016towards}. A linear SVM classifier is trained on the $19$ samples from each of the $20$ classes in the evaluation dataset and tested on the remaining $1$ sample from each class. We perform the classification $20$ times, leaving out a different sample from each class, in each experiment. The results are shown in Table \ref{omni_class}. For comparison, we also show classification scores when PCA is used as a feature extractor instead of a learned encoder.



Results show that the the DAAE model out-performs the AAE model, while the iDAAE performs less well, suggesting that features learned by the DAAE transfer better to new tasks, than those learned by the iDAAE. The AAE, iDAAE and DAAE models also out-perform PCA.

\begin{table}
\centering
\caption{\textbf{Omniglot Classification on all $964$ Test Set Classes and On $20$ Evaluation Classes.} }
\begin{tabular}{lcccc}\toprule
    
    Model & Test Acc. & Eval Acc. \% \\ \midrule
    AAE & 18.36\% & 78.75 \%\\
    DAAE & 31.74\% & \textbf{83.00}\%\\
    iDAAE & \textbf{34.02}\% & 78.25\%\\ \hdashline
    PCA & 31.02\% & 76.75\%\\ \hdashline
    Random chance & 0.11\% & 5\%\\
    \bottomrule
\end{tabular}
\label{omni_class}
\end{table}


\subsubsection{Classification: CelebA} 
\label{sec:class:celebA}


We perform a more extensive set of experiments to evaluate the \textbf{linear} separability of encodings learned on the celebA dataset and compare to state of the art methods including the VAE \cite{kingma2013auto} and the $\beta$-VAE\footnote{The $\beta$-VAE \cite{higgins2016beta} weights the $KL$ term in the VAE cost function with $\beta>1$ to encourge better organisation of the latent sapce, factorising the latent encoding into interpretable, indepenant compoenents}. \cite{higgins2016beta}.

We train a linear SVM on the encodings of a DAAE (or iDAAE) to predict labels for facial attributes, for example `Blond Hair', `No Beard' etc. . In our experiments, we compare classification accuracy on $12$ attributes obtained using the AAE, DAAE and iDAAE compared to previously reported results obtained for the VAE \cite{kingma2013auto} and the $\beta$-VAE \cite{higgins2016beta}, these are shown in Figure \ref{face:benchmark}. The results for the VAE and $\beta$-VAE were obtained using a similar approach to ours and were reported by Kumar et al. \cite{kumar2017variational}. We used the same hyper parameters to train all models and a fixed noise level of $\sigma=0.25$, the iDAAE was trained with $M=10$. The Figure (\ref{face:benchmark}) shows that the AAE, iDAAE and DAAE models outperform the VAE and $\beta$-VAE models on most facial attribute categories.


\begin{figure}
\centering
\includegraphics[width=1.0\columnwidth]{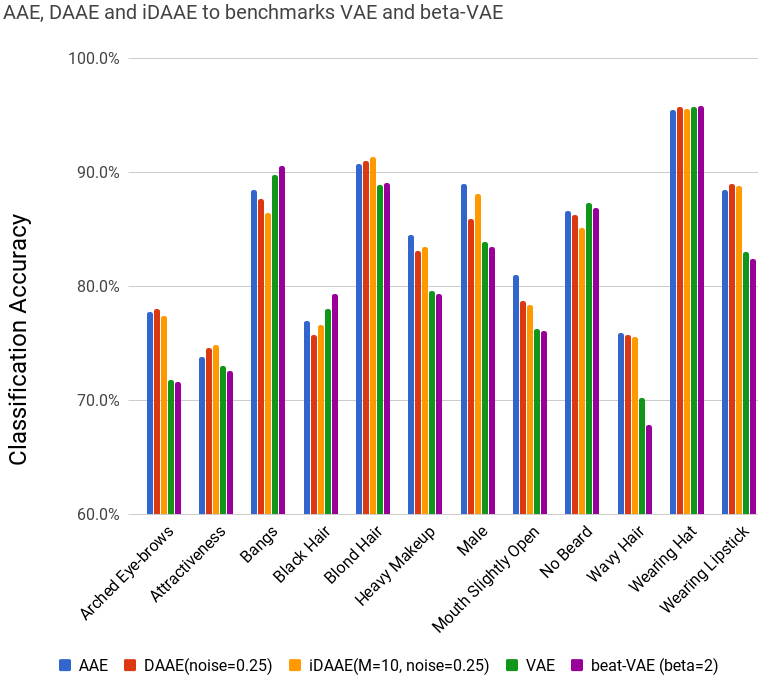} %
\caption{\textbf{Facial Attribute Classification} Comparison of classification scores for an AAE, DAAE, iDAAE compared to the VAE \cite{kingma2013auto} and $\beta$-VAE \cite{higgins2016beta}. A Linear SVM classifier is trained on encodings to demonstrate the linear separability of representation learned by each model. The attribute classification values for the VAE and $\beta$-VAE were obtained from Kumar et al. \cite{kumar2017variational}}
\label{face:benchmark}
\end{figure}

To compare models more easily, we ask the question, `On how many facial attributes does one model out perform another?', in the context of facial attribute classification. We ask this question for various combinations of model pairs, the results are shown in Figure \ref{faces:pi}. Figures \ref{faces:pi} (a) and (b), comparing the AAE to DAAE and iDAAE respectively, demonstrate that for some attributes the denoising models outperform the non-denoising models. More over, the particular attributes for which the DAAE and iDAAE outperform the AAE is (fairly) consistent, both DAAE and iDAAE outperform the AAE on the (same) attributes: `Attractive', `Blond Hair', `Wearing Hat', `Wearing Lipstick'. The iDAAE outperforms on an additional attribute, `Arched Eyebrows'.

\begin{figure}
\centering
\begin{subfigure}{0.24\columnwidth}
\includegraphics[width=0.9\columnwidth]{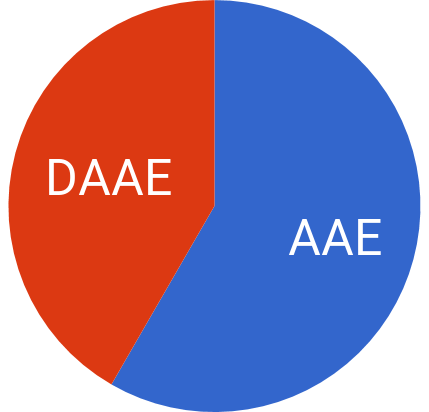}
\centering
\caption{}
\end{subfigure}
\begin{subfigure}{0.24\columnwidth}
\includegraphics[width=0.9\columnwidth]{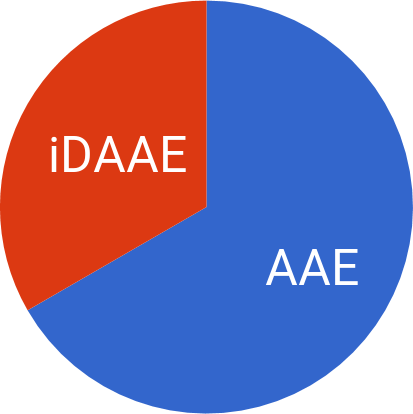}
\centering
\caption{}
\end{subfigure}
\begin{subfigure}{0.24\columnwidth}
\includegraphics[width=0.9\columnwidth]{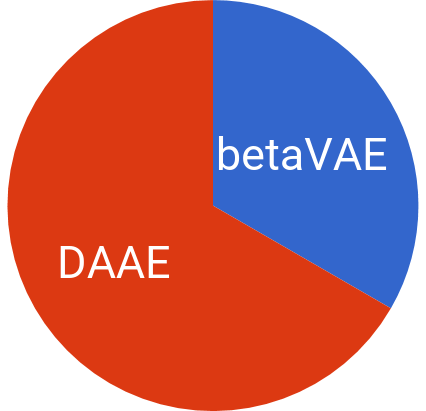}
\centering
\caption{}
\end{subfigure}
\begin{subfigure}{0.24\columnwidth}
\includegraphics[width=0.9\columnwidth]{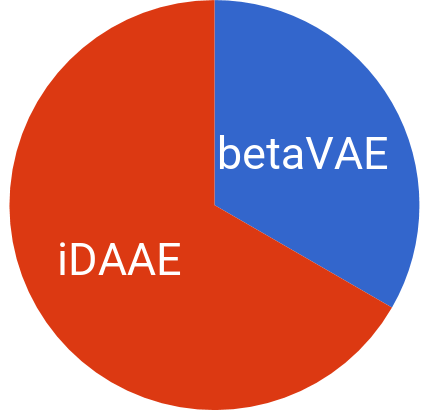}
\centering
\caption{}
\end{subfigure}
\caption{\textbf{On how many facial attributes does one model out perform another?} For each chart, each portion shows the number of facial attributes that each model out performs the other model in the same chart.}
\label{faces:pi}
\end{figure}


There are various hyper parameters that may be chosen to train these models; for the DAAE and iDAAE we may choose the level of corruption and for the iDAAE we may additionally choose the number of integration steps, $M$ used during training. We compare attribute classification results for $3$ vastly different choices of parameter settings. The results are presented as a bar chart in Figure \ref{faces:robust} for the DAAE. Additional results for the iDAAE are shown in the Appendix (Figure \ref{faces:robust_iDAAE}). These figures show that the models perform well under various different parameter settings. Figure \ref{faces:robust} suggests that the model performs better with a smaller amount of noise $\sigma=\{0.1, 0.25\}$ rather than with $\sigma=1.0$, however it is important to note that a large amount of noise does not `break' the model. These results demonstrate that the model works well for various hyper parameters, and fine tuning is not necessary to achieve reasonable results (when compared to the VAE for example). It is possible that further fine tuning may be done to achieve better results, however a full parameter sweep is highly computationally expensive.



\begin{figure}
\centering
\includegraphics[width=1.0\columnwidth]{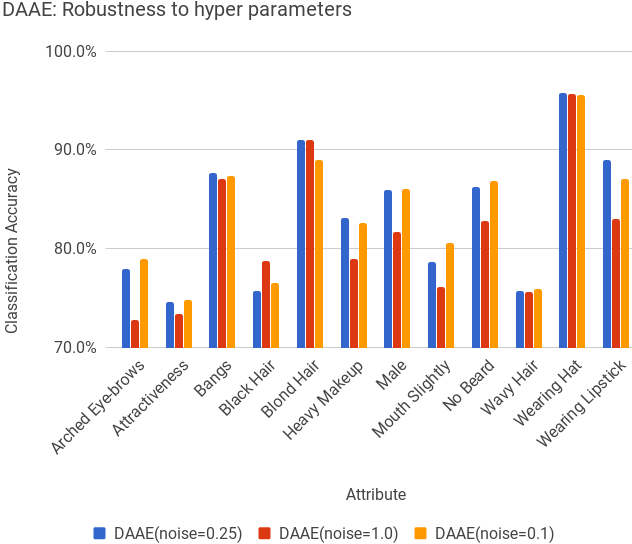}
\caption{\textbf{DAAE Robustness to hyper parameters} }
\label{faces:robust}
\end{figure}

From this section, we may conclude that with the exception of $3$ facial attributes, AAEs and variations of AAEs are able to outperform the VAE and $\beta$-VAE on the task of facial attribute classification. This suggests that AAEs and their variants are interesting models to study in the setting of learning linearly separable encodings. We also show that for a \textbf{specific set} of several facial attribute categories, the iDAAE or DAAE performs better than the AAE. This consistency, suggests that there are some specific attributes that the denoising variants of the AAE learn better than the non-denoising AAE.

\subsection{Trade-offs in Performance}  
The results presented in this section suggest that both the DAAE and iDAAE out-perform AAE models on most generation and some reconstruction tasks and suggest it is sometimes beneficial to incorporate denoising into the training of adversarial autoencoders. However, it is less clear which of the two new models, DAAE or iDAAE, are better for classification. When evaluating which one to use, we must consider both the practicalities of training, and for generative purposes, the practicalities -- primarily computational load -- of each model.

The integrating steps required for training an iDAAE means that it may take longer to train than a DAAE. On the other hand, it is possible to perform the integration process in parallel provided that sufficient computational resource is available. Further, once the model is trained, the time taken to compute encodings for classification is the same for both models. Finally, results suggest that using as few as $M=5$ integrating steps during training, leads to an improvement in classification score. This means that for some classification tasks, it may be worthwhile to train an iDAAE rather than a DAAE.


For generative tasks, neither the DAAE nor the iDAAE model consistently out-perform the other in terms of log-likelihood of synthesized samples. The choice of model may be more strongly affected by the computational effort required during training or sampling. In terms of log-likelihood on the synthesized samples, an iDAAE using even a small number of integration steps ($M=5$) during training of an iDAAE leads to better quality images being generated, and similarly using even one step of sampling with a DAAE leads to better generations.


Conflicting log-likelihood values of generated samples between testing and evaluation datasets means that these measurements are not a clear indication of how the number of sampling iterations affects the visual quality of samples synthesized using a DAAE. In some cases it may be necessary to visually inspect samples in order to assess effects of multiple sampling iterations (Figure \ref{omni_chain_sample}). 


\section{Conclusion}




We propose two types of denoising autoencoders, where a posterior is shaped to match a prior using adversarial training. In the first, we match the posterior conditional on corrupted data samples to the prior; we call this model a DAAE. In the second, we match the posterior, conditional on original data samples, to the prior. We call the second model an integrating DAAE, or iDAAE, because the approach involves using Monte Carlo integration during training.

Our first contribution is the extension of adversarial autoencoders (AAEs) to denoising adversarial autoencoders (DAAEs and iDAAEs). Our second contribution includes identifying and addressing challenges related to synthesizing data samples using DAAE models. We propose synthesizing data samples by iteratively sampling a DAAE according to a MC transition operator, defined by the learned encoder and decoder of the DAAE model, and the corruption process used during training. 

Finally, we present results on three datasets, for three tasks that compare both DAAE and iDAAE to AAE models. The datasets include: handwritten characters (Omniglot \cite{lake2015human}), a collection of human-like sprite characters (Sprites \cite{reed2015deep}) and a dataset of faces (CelebA \cite{liu2015faceattributes}). The tasks are reconstruction, classification and sample synthesis. 



\section*{Acknowledgment}

We acknowledge the Engineering and Physical Sciences Research Council for funding through a Doctoral Training studentship. We would also like to thank Kai Arulkumaran for interesting discussions and managing the cluster on which many experiments were performed. We also acknowledge Nick Pawlowski and Martin Rajchl for additional help with clusters.

\bibliographystyle{abbrv}
\bibliography{bib}



%
\clearpage
\newpage

\appendices

\section{Proofs}

\setcounter{lemma}{0}

\begin{lemma}  \label{p(x)} $\mathcal{P}(z)$ is a stationary distribution for the Markov chain defined by the sampling process in \ref{markovChain}. 
\end{lemma}

\begin{proof}
Consider the case where $z^{(0)} \sim \mathcal{P}(z)$. $x^{(0)} \sim p_\theta(x|z^{(0)})$ is from $\mathcal{P}(x)$, by equation \ref{assumption1}. Following the sampling process, $\tilde{x}^{(0)} \sim c(\tilde{x}|x^{(0)})$, $z^{(0)} \sim q_\phi(z|\tilde{x}^{(0)})$, $z^{(1)}$ is also from $\mathcal{P}(z)$, by equation \ref{pz}. 
\textit{Similar to proof in Bengio \cite{bengio2014deep}}. Therefore $\mathcal{P}(z)$ is a stationary distribution of the Markov chain defined by \ref{markovChain}.
\end{proof}

\begin{lemma} \label{ergodic} The Markov chain defined by the transition operator, $T_{\theta,\phi}(z_{t+1}|z_t)$ (\ref{eqn:T})is ergodic, provided that the corruption process is additive Gaussian noise and that adversarial pair, $q_\phi(z|\tilde{x})$ and $d_\chi(z)$ are optimal within the adversarial framework.
\end{lemma}

\begin{proof}
Consider $X=\{x: \mathcal{P}(x)>0\}$, $\tilde{X}=\{\tilde{x} : c(\tilde{x}|x)>0 \}$ and $Z=\{z: \mathcal{P}(z)>0\}$. Where $Z \subseteq \{ z :p(z)>0 \}$ and $X \subseteq \tilde{X}$:
\begin{enumerate}
    \item Assuming that $p_\theta(x|z)$ is a good approximation of the underlying probability distribution, $\mathcal{P}(x|z)$, then $\forall x_j \sim \mathcal{P}(x)$  $\exists$ $z_i \sim \mathcal{P}(z)$ s.t. $p_\theta(x_j | z_i)>0$. 
    \item Assuming that adversarial training has shapped the distribution of $q_\phi(z|\tilde{x})$ to match the prior, $p(z)$, then $\forall z_i \sim p(z)$ $\exists$ $\tilde{x}_j$ s.t. $q_\phi(z_i|\tilde{x}_j)>0$.
    This holds because if not all points in $p(z)$ could be visited, $q_\phi(z|\tilde{x})$ would not have matched the prior. 
    \end{enumerate}
1) suggests that every point in $X$ may be reached from a point in $Z$ and 2) suggests that every point in $Z$ may be reached from a point in $\tilde{X}$.
Under the assumption that $c(\tilde{x}|x)$ is an additive Gaussian corruption process then $\tilde{x}_i$ is likely to lie within a (hyper) spherical region around $x_i$. If the corruption process is sufficiently large such that (hyper) spheres of nearby $x$ samples overlap, for an $x_i$ and $x_{i+m}$ $\exists$ a set $\{x_{i+1},...,x_{i+m-1}\}$ such that, $ \sup(c(\tilde{x}|x_i)) \cap \sup (c(\tilde{x}|x_{i+1})) \neq \emptyset, \forall i=1,...,(m-1)$ and where $\sup$ is the support.
Then, it is possible to reach any $z_i$ from any $z_j$ (including the case $j=i$). Therefore, the chain is both irreducible and positive recurrent.

To be ergodic the chain must also be aperiodic:
between any two points $x_i$ and $x_j$, there is a boundary, where $x$ values between $x_i$ and the boundary are mapped to $z_i$, and points between $x_j$ and the boundary are mapped to $z_j$. By applying the corruption process to $x^{(t)}=x_i$, followed by the reconstruction process, there are always at least two possible outcomes because we assume that all (hyper) spheres induced by the corruption process overlap with at least one other (hyper) sphere: either $\tilde{x}^{(t)}$ is not pushed over the boundary and $z^{(t+1)}=z_i$ remains the same, or $\tilde{x}^{(t)}$ is pushed over the boundary and $z^{(t+1)}=z_j$ moves to a new state. The probability of either outcome is positive, and so there is always more that one route between two points, thus avoiding periodicity provided that for $x_i \neq x_j$, $z_i \neq z_j$. 
Consider, the case where for $x_i \neq x_j$, $z_i=z_j$, then it would not be possible to recover both $x_i$ and $x_j$ using $\mathcal{P}(x|z)$ and so if $\mathcal{P}(x_i|z)>0$ then $\mathcal{P}(x_j|z)=0$ (and vice verse), which is a contradiction to 1).




\end{proof}





\section{Examples Of Sythesised Samples}
\begin{figure}[H]
    \centering
    \includegraphics[width=\columnwidth]{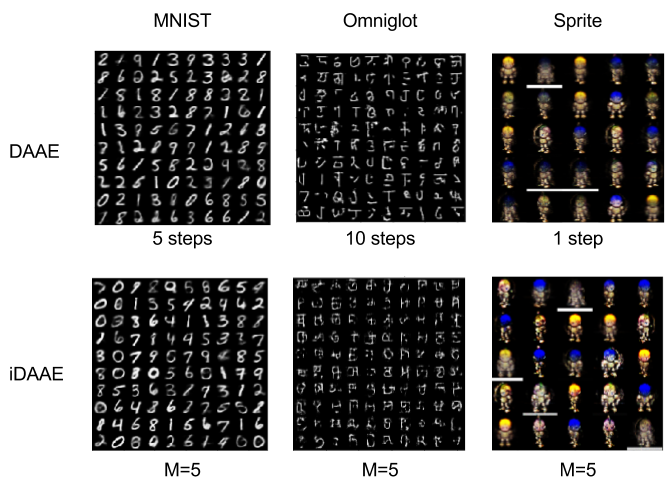}
    \caption{\textbf{Examples of synthesised samples:} Examples of randomly sytheisised data samples.}
    \label{best_gen}
\end{figure}

\section{Algorithm For Training the iDAAE}
\begin{algorithm}
\caption{iDAAE : Matching $\tilde{q}_\phi(z|x)$ to $p(z)$}
\label{algv1}
$\mathbf{x}=\{ x_1, x_2, x_{N-1}\} \sim p(x)$ {\color{g} \#draw a batch of samples from the training data}\\
\For{$k=1$ to NoEpoch}
{
    $\tilde{\mathbf{x}}=C(\mathbf{x})$ {\color{g}\#corrupt all samples}\\
    $\mathbf{z}=E_\phi(\mathbf{\tilde{x}})$ {\color{g}\#encode all corrupted samples}\\
    $\mathbf{\hat{x}}=R_\theta(\mathbf{z})$ {\color{g}\#reconstruct}\\
    \vspace{2mm}
    {\color{g} \#minimise reconstruction cost}\\
    $\mathcal{L}_{rec} = -(\mathbf{\hat{x}}\log \mathbf{x} + (1-\mathbf{\hat{x}})\log (1-\mathbf{x})).mean()$\\
    $\phi \leftarrow \phi - \alpha \nabla_\phi \mathcal{L}_{rec}$\\
    $\theta \leftarrow \theta - \alpha \nabla_\phi \mathcal{L}_{rec}$\\
    \vspace{2mm}
    {\color{g} \#match $\tilde{q}_\phi(z|x)$ to $p(z)$ using adversarial training}\\
    $\mathbf{z}_{fake}=$approx\_z($\mathbf{x}$)  {\color{g} \#draw samples for $\tilde{q}_\phi(z|x)$}\\
    $\mathbf{z}_{real} \sim p(z)$ {\color{g} \#draw samples from prior $p(z)$}\\
    {\color{g} \#train the discriminator}\\
    $\mathcal{L}_{dis}=-\frac{1}{N} \sum_{i=0}^{N-1} \log D_\chi(\mathbf{z}_{{real}_i})$\\ $-\frac{1}{N} \sum_{i=0}^{N-1} \log (1-D_\chi(\mathbf{z}_{{fake}_i})) $ \\
    $\chi \leftarrow \chi - a \nabla_\chi \mathcal{L}_{dis}$\\
    {\color{g} \#train the decoder to match the prior}\\
    $\mathcal{L}_{prior} = \frac{1}{N} \sum_{i=0}^{N-1} \log(1-D_\chi( \mathbf{z}_{{fake}_i}))$\\
    $\phi \leftarrow \phi - \alpha \nabla_\phi$   
}
\end{algorithm}

\section{Algorithm For Monte Carlo Integration}
\begin{algorithm}
\caption{Drawing samples from $\tilde{q}_\phi(z|x)$}
\label{algz}
\underline{function: approx\_z($\{x_1, x_2,...x_{N-1}\}$)}\\
\For{$i=0$ to $N-1$}
{   
    $\hat{z_i}=[$ $]$\\
    \For{$j=1$ to $M$}
    {
        $\tilde{x}_{i,j} = C(x_i)$\\
        $z_{i,j} = E_\phi(\tilde{x}_j)$\\    
    }
    $\hat{z_i}.append(\frac{1}{M}\sum_{j=1}^M z_j)$ \\  
}
\textbf{return} $\mathbf{\hat{z}}=\{\hat{z}_1,\hat{z}_2...\hat{z}_{N-1}\}$
\end{algorithm}

\section{MNIST Results}

\subsection{Datasets: MNIST}
The MNIST dataset consists of grey-scale images of handwritten digits between $0$ and $9$, with $50$k training samples, $10$k validation samples and $10$k testing samples. The training, validation and testing dataset have an equal number of samples from each category. The samples are $28$-by-$28$ pixels. The MNIST dataset is a simple dataset with few classes and many training examples, making it a good dataset for proof-of-concept. However, because the dataset is very simple, it does not necessarily reveal the effects of subtle, but potentially important, changes to algorithms for training or sampling. For this reason, we consider two datasets with greater complexity.

\subsection{Architecture and Training: MNIST}
For the MNIST dataset, we train a total of $5$ models detailed in Table \ref{MNIST_models}. The encoder, decoder and discriminator networks each have two fully connected layers with $1000$ neurons each. For most models the size of the encoding is $10$ units, and the prior distribution that the encoding is being matched to is a $10$D Gaussian.
All networks are trained for $100$ epochs on the training dataset, with a learning rate of $0.0002$ and a batch size of $64$. The standard deviation of the additive Gaussian noise used during training is $0.5$ for all iDAAE and DAAE models.

An additional DAAE model is trained using the same training parameters and networks as described above but with a mixture of $10$ $2$D Gaussians for the prior. Each $2$D Gaussian with standard deviation $0.5$ is equally spaced with its mean around a circle of radius $4$ units. This results in a prior with $10$ modes separated from each other by large regions of very low probability. This model of the prior is very unrealistic, as it assumes that MNIST digits occupy distinct regions of image probability space. In reality, we may expect two numbers that are similar to exist side by side in image probability space, and for there to exist a smooth transition between handwritten digits. As a consequence of this, this model is intended specifically for evaluating how well the posterior distribution over latent space may be matched to a mixture of Gaussians - something that could not be achieved easily by a VAE \cite{kingma2013auto}.


\begin{table}
\centering
\caption{\textbf{Models trained on MNIST:} $5$ models are trained on the MNIST dataset. \textit{Corruption} indicates the standard deviation of Gaussian noise added during the corruption process, $c(\tilde{x}|x)$. \textit{Prior} indicated the prior distribution imposed on the latent space. \textit{M} is the number of Monte Carlo integration steps (see Algorithm \ref{algz} in Appendix) used during training - this applies only to the iDAAE.}
\begin{tabular}{ccccc}
\toprule
ID & Model & Corruption & Prior &$M$\\ \midrule
1 & AAE & 0.0 & 10D Gaussian & - \\
2 & DAAE & 0.5 & 10D Gaussian & - \\
3 & DAAE & 0.5 & 10-GMM & - \\
4 & iDAAE & 0.5 & 10D Gaussian & 5 \\
5 & iDAAE & 0.5 & 10D Gaussian & 25 \\
\bottomrule
\label{MNIST_models}
\end{tabular}
\end{table}

\subsection{Reconstruction: MNIST}
Table \ref{MNIST_rec} shows reconstruction error on samples from the testing dataset, for each of the models trained on the MNIST training dataset. Reconstruction error for the DAAE and iDAAE trained with a 10D Gaussian prior do not out-perform the AAE. However, the reconstruction task for the AAE model was less challenging than that of the DAAE and iDAAE models because samples were not corrupted before the encoding step. The best model for reconstruction was the iDAAE with $M=5$, increasing $M$ to $25$ led to worse reconstruction error, as reconstructions tended to appear more like mean samples. 

Reconstruction error for the DAAE trained using a 2D mixture of 10 Gaussians under-performed compared to the rest of the models. All reconstructions looked like mean images, which may be expected given the nature of the prior. 

\begin{table}
\centering
\caption{\textbf{MNIST Reconstruction:} \textit{Recon.} shows the mean squared error for reconstructions of corrupted test data samples accompanied by the standard error. \textit{Corruption} is the standard deviation of the additive Gaussian noise used during training and testing.}
\begin{tabular}{lcccc}\toprule
    \multicolumn{4}{c}{Model \& Training} & Recon.\\ \cmidrule(lr){1-4}\cmidrule(lr){5-5}
    Model & Corruption & Prior & $M$ & Mean $\pm$ s.e. \\\cmidrule(lr){1-1} \cmidrule(lr){2-2} \cmidrule(lr){3-3} \cmidrule(lr){4-4} \cmidrule(lr){5-5} 
    AAE & 0.0 & 10D Gaussian &-& 0.017$\pm$0.001\\
    DAAE& 0.5 & 10D Gaussian & - &0.023$\pm$0.001\\
    DAAE& 0.5 & 2D 10-GMM &-& 0.043$\pm$0.001\\
    iDAAE&0.5&10D Gaussian& 5 & 0.022$\pm$0.001\\
    iDAAE&0.5&10D Gaussian& 25 & 0.026$\pm$0.001\\
    \bottomrule
\end{tabular}
\label{MNIST_rec}
\end{table}

\subsection{Sampling: MNIST}
\label{sampling_mnist}
To calculate the log-likelihood of samples drawn from DAAE, iDAAE and AAE models trained on the MNIST dataset, a Parzen window is fitted to $1\times 10^4$ generated samples from a single trained model. The bandwidth for constructing the Parzen window was selected by choosing one of $10$ values, evenly space along a log axis between $-1$ and $0$, that maximises the likelihood on a validation dataset. The log-likelihood for each model was then evaluated on the test dataset. 

Table \ref{MNIST_ll} shows the log-likelihood on the test dataset for samples drawn from models trained on the MNIST training dataset. First, we will discuss models trained with a $10$D Gaussian prior. The best log-likelihood was achieved by the iDAAE with $M=5$. Synthesised samples generated by MC sampling from the DAAE model caused the log-likelihood to decrease. This may be because the samples tended towards mean samples, dropping modes, causing the log-likelihood to decrease. Initial samples and those obtained after $5$ iterations of sampling are shown in Figure \ref{mnist_mcmc_200}. The samples in \ref{mnist_mcmc_200}(d) are clearer than in \ref{mnist_mcmc_200}(c). Although mode-dropping is not immediately apparent, note that the digit 4 is not present after 5 iterations.


Now, we consider the DAAE model trained with a $2$D, $10$-GMM prior. Samples drawn from the prior are shown in Figure \ref{mnist_mcmc}(a). The purpose of this experiment was to show the effects of MC sampling, where the distribution from which initial  samples of $z_0$ are drawn is significantly different to the prior. Samples of $z_0$ were drawn from a normal distribution and passed through the decoder of the DAAE model to produce initial image samples, see Figure \ref{mnist_mcmc}(c). A further $9$ steps of MC sampling were applied to synthesise the samples shown in Figure \ref{mnist_mcmc}(d). As expected, the initial image samples do not look like MNIST digits, and MC sampling improves samples dramatically. Unfortunately however, many of the samples appear to correspond to samples at modes of the data distribution. In addition, several modes appear to be missing from the model distribution. This may be attributed to the nature of the prior, since we did not encounter this problem to the same extent when using a $10$D Gaussian prior (see Figure \ref{mnist_mcmc_200}).

Synthesising MNIST samples is fairly trivial, since there are many training examples and few classes. For this reason, it is difficult to see the benefits of using DAAE or iDAAE models compared to AAE models. Now, we focus on the two more complex datasets, Omniglot and Sprites. We find that iDAAE models and correctly sampled DAAE models may be used to synthesise samples with higher log-likelihood than samples synthesised using an AAE.


\begin{table}
\centering
\caption{\textbf{MNIST log-likelihood of $p_\theta(x)$:} To calculate the log-likelihood of $p_\theta(x)$, a Parzen window was fit to $10^4$ generated samples and the mean likelihood was reported for a testing data set. The bandwidth used for the Parzen window was determined using a validation set. The training, test and validation datasets had different samples. }
\begin{tabular}{lccccc}
    \toprule
    \multicolumn{4}{c}{Model \& Training} & \multicolumn{2}{c}{Log-likelihood} \\ \cmidrule(lr){1-4}\cmidrule(lr){5-6}
    Model & Corruption & Prior & $M$ & $x^{(0)}$ & $x^{(5)}$\\\cmidrule(lr){1-1} \cmidrule(lr){2-2} \cmidrule(lr){3-3} \cmidrule(lr){4-4} \cmidrule(lr){5-5} \cmidrule(lr){6-6}
    AAE & 0.0 & 10D Gaussian & - & 529 & - \\
    DAAE & 0.5 & 10D Gaussian & - & 529 & 444 \\
    DAAE & 0.5 & 2D 10-GMM & - & 9 & 205 \\ 
    iDAAE & 0.5 & 10D Gaussian & 5 & 532 & -\\
    iDAAE & 0.5 & 10D Gaussian & 25 & 508 & - \\
    \hline
\end{tabular}
\label{MNIST_ll}
\end{table}

\begin{figure}
\centering
\begin{subfigure}{0.4\columnwidth}
\includegraphics[width=0.8\columnwidth, height=0.9\columnwidth]{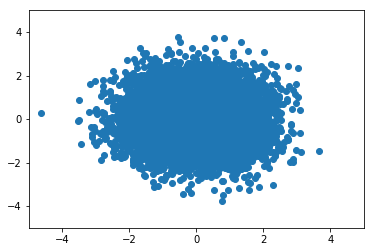}
\caption{Prior (one channel)}
\end{subfigure}
\begin{subfigure}{0.4\columnwidth}
\includegraphics[width=0.9\columnwidth, height=0.9\columnwidth]{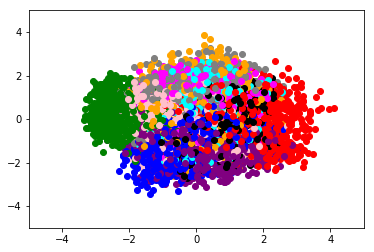}
\caption{Latent samples}
\end{subfigure}
\begin{subfigure}{0.4\columnwidth}
\includegraphics[width=0.9\columnwidth, height=0.9\columnwidth]{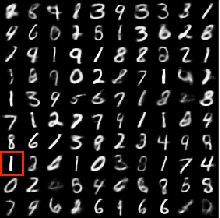}
\caption{$x^{(0)}$}
\end{subfigure}
\begin{subfigure}{0.4\columnwidth}
\includegraphics[width=0.9\columnwidth]{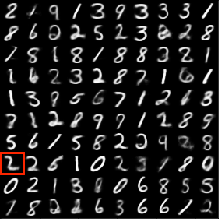}
\caption{$x^{(5)}$}
\end{subfigure}
\caption{\textbf{MNIST Markov chain (MC) sampling: One channel of the 10D Gaussian Prior} a) One channel of the 10D Gaussian prior used to train the DAAE. b) Samples drawn from $q_\phi(z|\tilde{x})$, projected in to $2$D space, and colour coded by the digit labels. c) Initial samples, $x^{(0)}$ generated by a normal distribution. d) \textbf{Corresponding} samples after $5$ iterations of MC sampling. Notice how the highlighted ``$1$" changes to a ``$2$" after $5$ iterations of MC sampling.}
\label{mnist_mcmc_200}
\end{figure}

\begin{figure}
\centering
\begin{subfigure}{0.45\columnwidth}
\includegraphics[width=0.9\columnwidth, height=0.9\columnwidth]{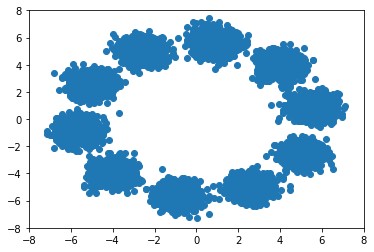}
\caption{2D Gaussian Mixture prior}
\label{}
\end{subfigure}
\begin{subfigure}{0.45\columnwidth}
\includegraphics[width=0.9\columnwidth, height=0.9\columnwidth]{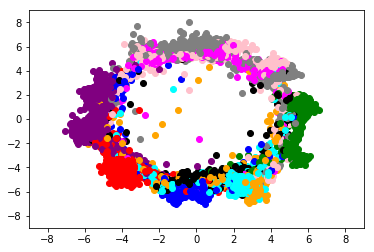}
\caption{Latent samples }
\label{}
\end{subfigure}
\begin{subfigure}{0.45\columnwidth}
\includegraphics[width=0.9\columnwidth]{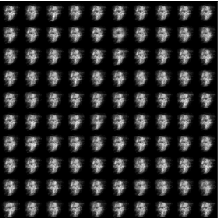}
\caption{$x^{(0)}$}
\label{}
\end{subfigure}
\begin{subfigure}{0.45\columnwidth}
\includegraphics[width=0.9\columnwidth]{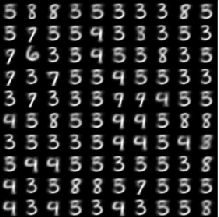}
\caption{$x^{(9)}$}
\label{}
\end{subfigure}
\caption{\textbf{MNIST Markov chain (MC) sampling:} a) 2D Gaussian Mixture prior used to train a DAAE. b) Samples drawn from $q_\phi(z|\tilde{x})$ projected into 2D space using PCA and colour coded by the digit labels. c) Initial sample, $x^{(0)}$, generated by a normal distribution. d) \textit{Corresponding} samples generated after 9 further iterations of MC sampling. }
\label{mnist_mcmc}
\end{figure}

\subsubsection{Classification: MNIST}
The MNIST dataset consists of $10$ classes, $[0,9]$, the classification task involves correctly predicting a label in this interval. For the MNIST dataset, the SVM classifier is trained on encoded samples from the MNIST training dataset and evaluated on encoded samples from the MNIST testing dataset, results are shown in Table \ref{MNIST_class}. 

First, we consider the results for DAAE, iDAAE and AAE models trained with a $10$D Gaussian prior. Classifiers trained on encodings extracted from the encoders of trained DAAE, iDAAE or AAE models out-performed classifiers trained on PCA of image samples. Classifiers trained on the encodings extracted from the encoders of learned DAAE and iDAAE models out-performed those trained on the encodings extracted from the encoders of the AAE model.

The differences in classification score for each model on the MNIST dataset are small; this might be because it is relatively easy to classify MNIST digits with very high accuracy \cite{simard2003best}. We turn now to a more complicated dataset, the Omniglot dataset, to better show the benefits of using denoising when training adversarial autoencoders.

\begin{table}
\centering
\caption{\textbf{MNIST Classification:} SVM classifiers with RBF kernels were trained on encoded MNIST training data samples. The samples were encoded using the encoder of the trained AAE, DAAE or iDAAE models. Classification scores are given for the MNIST test dataset.}
\begin{tabular}{lcccc}\toprule
    \multicolumn{4}{c}{Model \& Training} & \multirow{2}{*}{Accuracy}\\ \cmidrule(lr){1-4}
    Model & Corruption & Prior & $M$ & \% \\\cmidrule(lr){1-1} \cmidrule(lr){2-2} \cmidrule(lr){3-3} \cmidrule(lr){4-4} \cmidrule(lr){5-5} 
    AAE & 0.0 & 10D Gaussian &-& 95.19\%\\
    DAAE & 0.5 & 10D Gaussian & - &95.28\%\\
    DAAE & 0.5 & 2D 10-GMM &-& 73.81\%\\
        iDAAE&0.5&10D Gaussian& 5 & 96.48\%\\
        iDAAE&0.5&10D Gaussian & 25 & 96.24\%\\\hdashline[.4pt/1pt]
    \multicolumn{4}{l}{PCA + SVM} & 94.73\% \\
    \bottomrule
\end{tabular}
\label{MNIST_class}
\end{table}

\section{Details for calculating log-likelihood}
\label{sec:ll}

\subsection{Omniglot}
To calculate log-likelihood of samples, a Parzen window was fit to $1\times 10^3$ synthesised samples, where the bandwidth was determined on the testing dataset in a similar way to that in Section \ref{sampling_mnist}. The log-likelihood was evaluated on both the evaluation dataset, and the testing dataset. To compute the log-likelihood on the of the testing dataset a Parzen window was fit to a new set of synthesized samples, different to those used to calculate the bandwidth. The results are shown in Figure \ref{omni_ll}.

\subsection{Sprites}
To calculate log-likelihood of samples, a Parzen window was fit to $1 \times 10^3$ synthesized samples. The bandwidth was set as previously described in Section \ref{sampling_mnist}; we found the optimal bandwidth to be 1.29.

\section{iDAAE Robustness to hyper parameters}
\begin{figure}
\centering
\includegraphics[width=0.9\columnwidth]{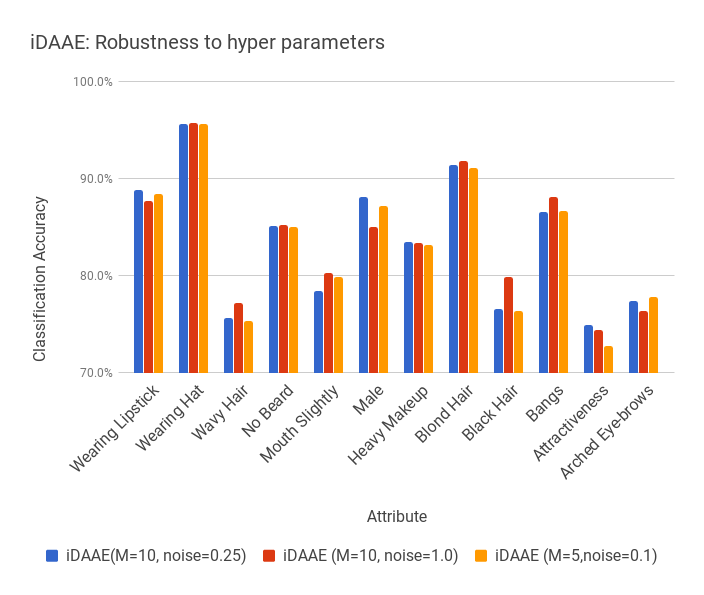}
\caption{iDAAE}
\caption{\textbf{Robustness to hyper parameters}}
\label{faces:robust_iDAAE}
\end{figure}

\end{document}